\documentclass[english]{article}

\usepackage{geometry}
\usepackage[T1]{fontenc}
\usepackage[utf8]{inputenc}
\usepackage{bm}
\usepackage{amsmath}
\usepackage{amssymb} 
\usepackage[unicode=true,
 bookmarks=false, 
 breaklinks=false,pdfborder={0 0 1},colorlinks=false]
 {hyperref}
\hypersetup{
 colorlinks,citecolor=blue,filecolor=blue,linkcolor=blue,urlcolor=blue}
\usepackage{xcolor,colortbl}
\definecolor{Gray}{gray}{0.85}
\usepackage{enumitem}
\makeatletter

\usepackage{amsthm}
\usepackage{cite}  
\usepackage{comment}
\usepackage[sort,compress]{natbib}
\usepackage{booktabs,mathtools}
\usepackage{graphicx}
\usepackage{subcaption}
\usepackage{algorithm}
\usepackage{algorithmic}
\usepackage{multirow}
\usepackage{dsfont}
\usepackage{color}
\usepackage{float}
\usepackage[capitalize]{cleveref}
\crefname{assumption}{Assumption}{Assumptions}
\usepackage{url}
\usepackage{tcolorbox}
\usepackage{changepage}
\definecolor{laixi}{RGB}{138,43,226}
\definecolor{chengrui}{RGB}{200,0,0}

\newcommand{\revar}{\mathrm{Var}_{p}}
\newcommand{\tmvar}{\mathrm{Var}_{\tilde{p}^t}}
\newcommand{\tarvar}{\mathrm{Var}_{p_{\mathrm{tar}}}}

\newcommand{\nt}{n^t(s,a)}
\newcommand{\srcn}{n_{\mathrm{src}}(s,a)}
\newcommand{\sa}{\mathcal{S}\times\mathcal{A}}
\newcommand{\overnt}{\overline{n}^t(s,a)}
\newcommand{\tarovernt}{\overline{n}_{\mathrm{tar}}^t(s,a)}
\newcommand{\ba}{\mathbf{a}}
\renewcommand{\P}{\mathbb{P}}
\newcommand{\indic}[1]{\mathds{1}\left\{#1\right\}}

\newcommand{\dtv}[2]{\mathrm{TV}(#1,#2)}

\newcommand{\tnt}{\tilde{n}^t}
\newcommand{\tpt}{\tilde{p}^t}
\newtheorem{lemma}{Lemma}

\newtheorem{theorem}{Theorem}
\newtheorem{assumption}{Assumption}
\newtheorem{definition}{Definition}

\newtheorem{remark}{Remark}

\allowdisplaybreaks

\title{Hybrid Transfer Reinforcement Learning: \\Provable Sample Efficiency from Shifted-Dynamics Data}
 \author{
 	Chengrui Qu\thanks{College of Engineering, Peking University, Beijing, 100871, China.}\\
 	PKU 
 	\and
 	Laixi Shi\thanks{Department of Computing Mathematical Sciences, California Institute of Technology, CA 91125, USA.}\\
 	Caltech 
	\and
	Kishan Panaganti\footnotemark[2] \\
    Caltech\\
	\and
	Pengcheng You\footnotemark[1]  \\
	PKU \\
	\and
	Adam Wierman\footnotemark[2] \\
	Caltech
 	} 
\date{\today}
\begin{document}

\maketitle
\begin{abstract}
   Online Reinforcement learning (RL) typically requires high-stakes online interaction data to learn a policy for a target task. This prompts interest in leveraging historical data to improve sample efficiency.
   The historical data may come from outdated or related source environments with different dynamics. It remains unclear how to effectively use such data in the target task to provably enhance learning and sample efficiency. To address this, we propose a hybrid transfer RL (HTRL) setting, where an agent learns in a target environment while accessing offline data from a source environment with shifted dynamics. We show that -- without information on the dynamics shift -- general shifted-dynamics data, even with subtle shifts, does not reduce sample complexity in the target environment.  However, with prior information on the degree of the dynamics shift, we design HySRL, a transfer algorithm that achieves problem-dependent sample complexity and outperforms pure online RL. Finally, our experimental results demonstrate that HySRL surpasses state-of-the-art online RL baseline.
\end{abstract}


\allowdisplaybreaks
\setcounter{tocdepth}{2}
\tableofcontents

\section{Introduction}
In online reinforcement learning (RL), an agent learns by continuously interacting with an unknown environment. While this approach has led to remarkable successes across various domains, such as robotics \citep{espeholt_impala_2018}, traffic control \citep{he_robust_2023} and game playing \citep{silver_mastering_2017}, it often requires billions of data from interactions to develop an effective policy \citep{li2023understandingcomplexitygainssingletask}. Moreover, in many real-world scenarios, such interactions can be costly, time-consuming, or unsafe \citep{eysenbach_off-dynamics_2021}, which significantly limits the broader application of RL in practice, highlighting the urgent need for more sample-efficient paradigms.

One promising direction to address sample inefficiency in RL is transfer learning \citep{zhu2023transferlearningdeepreinforcement}. When developing an effective policy for a target environment, it is often possible to leverage experiences from a similar source environment with shifted dynamics \citep{niu_comprehensive_2024}. These sources may include an imperfect simulator \citep{8460528}, historical operating data before external impacts \citep{Luo_Jiang_Yu_Zhang_Zhang_2022}, polluted offline datasets \citep{onlinepricing}, or data from other tasks in a multi-task setting \citep{sodhani2021multitaskreinforcementlearningcontextbased}. This concept has led to various domains and pipelines, such as meta RL \citep{finn2017modelagnosticmetalearningfastadaptation}, cross-domain RL \citep{eysenbach_off-dynamics_2021,liu2022daradynamicsawarerewardaugmentation}, and distributionally robust RL \citep{NEURIPS2023_fc8ee7c7}, which demonstrate varying levels of effectiveness.

From a practical standpoint, so far there are no clear signals on how to perform transfer learning sample-efficiently with theoretical guarantees. 
While some studies show that using shifted dynamics data can reduce the time required to achieve specific performance levels in the target environment \citep{liu2022daradynamicsawarerewardaugmentation,serrano2023similaritybasedknowledgetransfercrossdomain,10444921}, others indicate that sometimes these transfers hinder rather than help learning \citep{autonomous,pmlr-v180-you22a}, a phenomenon known as negative transfer.

These practical challenges highlight the need for theoretical insights, which have not been addressed in existing frameworks. Recently, a new stream of research called hybrid RL \citep{xie_policy_2022} has emerged, showing that, theoretically, an offline dataset with no dynamics shift can facilitate more efficient online exploration. However, when the dataset is collected from a source environment with shifted dynamics, it remains unclear whether this dataset can still enable more sample-efficient learning in the target environment. This brings us an interesting open question:
\begin{quote}\textit{Can data from a shifted source environment be leveraged to provably enhance sample efficiency when learning in a target environment?} 
\end{quote}

To answer this question, we formulate a problem called hybrid transfer RL, where an agent learns in a target environment while having access to an offline dataset collected from a source environment. The source and target environments differ only in transition uncertainties in the same world~\citep{DoshiVelez2013HiddenPM}.  Since these differences are typically unknown before exploring the target environment, we refer to them as an unknown dynamics shift. The learning goal is to find an optimal policy for the target environment using minimal interactions.

\paragraph{Contributions.}
In this work, we propose a hybrid transfer RL (HTRL) setting, where the source and target environments share the same world structure, only differing in their transitions. We first present a hardness result in terms of a minimax lower bound in general hybrid transfer RL and show provable sample efficiency gains from the source environment dataset with additional prior information. To the best of our knowledge, we are the first to look into the sample complexity of this transfer setting. Specifically:
\begin{itemize}[nosep,leftmargin=*]
    \item We formulate and focus on a new problem called hybrid transfer RL (HTRL). We find that even when the target MDP is similar to the source MDP in dynamics, data from the source MDP cannot reduce the sample complexity in the target MDP without further conditions, compared to state-of-the-art online RL sample complexity (\cref{theorem:lower-bound}). This result demonstrates that general HTRL is not feasible, motivating us to look into more practical yet meaningful settings.
    \item We study HTRL where prior information about the degree of the dynamics shift is available. A transfer algorithm, HySRL, is designed, which achieves a problem-dependent sample complexity that is at least as good as the state-of-the-art online sample complexity, offering sample efficiency gains in many scenarios (\cref{theorem:sample-complexity}). The key technical contributions involve extending the current reward-free and bonus-based exploration techniques to accommodate more general rewards and incorporating shifted-dynamics data into the analysis. In addition, we conduct experiments in the GridWorld environment to evaluate the proposed algorithm HySRL, demonstrating that HySRL achieves superior sample efficiency than the state-of-the-art pure online RL baseline.
\end{itemize}

\subsection{Related work}

\paragraph{Finite-sample analysis of online, offline, and hybrid RL.} Finite sample analysis in RL focuses on understanding the sample complexity -- how many samples are required to achieve a desired policy with certain performance. In this line of research, a non-exhaustive list in online RL includes \citet{dong2019q,zhang2020reinforcement,zhang2020model,jafarnia2020model,liu2020gamma,yang2021q,azar2017minimax,jin2018q,bai2019provably,zhang2020almost,menard2021ucb,domingues2021episodic,he2020nearly,zanette2019tighter,zhang2020reinforcement}, while offline RL has seen advances such as \citet{uehara2020minimax,li2014minimax,yang2020off,duan2020minimax,jiang2016doubly,jiang2020minimax,kallus2020double,duan2021optimal,xu2021unified,ren2021nearly,thomas2016data}, and hybrid RL frameworks are explored in \citet{song_hybrid_2023,xie_policy_2022,zhang_policy_2023,li_reward-agnostic_2023}. The most closely related setting is hybrid RL, in which an agent learns in a target environment with access to an offline dataset collected from the same environment. Our work extends hybrid RL by addressing cases where the offline dataset may also come from a related environment with shifted dynamics.

\paragraph{Transfer RL with dynamics shifts.} Cross-domain RL with dynamics shifts is the most related setting, which focuses on leveraging abundant samples from a source environment to reduce data requirements for a target environment \citep{eysenbach_off-dynamics_2021,liu2022daradynamicsawarerewardaugmentation,niu_when_2023,niu_h2o_2023,chen2024domainadaptationofflinereinforcement,wen_contrastive_2024}. A major challenge in these transfers is the dynamics shift between the source and target environments. Common approaches often involve training a classifier to distinguish between source and target transitions, by techniques such as combining source and target datasets for policy training \citep{wen_contrastive_2024,chen2024domainadaptationofflinereinforcement}, or reshaping rewards by introducing a penalty term for dynamics shifts \citep{eysenbach_off-dynamics_2021,liu2022daradynamicsawarerewardaugmentation}. While these methods show promising empirical performance, a systematic study on sample complexity is missing. Our work fills this gap by offering a novel theoretical perspective on cross-domain RL.

Other related transfer RL settings include distributionally robust offline RL \citep{pmlr-v130-zhou21d,panaganti22a,NEURIPS2023_fc8ee7c7,wang2024samplecomplexityofflinedistributionally,ma2023distributionallyrobustofflinereinforcement,liu2024minimaxoptimalcomputationallyefficient}, which focuses on training a robust policy using only an offline dataset, without further exploration, to optimize performance in the worst-case scenario of the target environment. Another area, meta RL \citep{finn2017modelagnosticmetalearningfastadaptation,duan2016rl2fastreinforcementlearning,wang2017learningreinforcementlearn, chen2022understandingdomainrandomizationsimtoreal,ye2023powerpretraininggeneralizationrl,mutti2024testtimeregretminimizationmeta}, trains an agent over a distribution of environments to enhance generalization capabilities. Our work contributes to distributionally robust offline RL by addressing the sample complexity when exploration in the target environment is allowed and complements meta RL by focusing on scenarios where the target environment lies outside the training distribution.

\paragraph{Reward-free RL.} Reward-free RL seeks to collect sufficient data to achieve optimality for any potential reward function. This paradigm was first proposed in \citet{jin2020rewardfreeexplorationreinforcementlearning}, with subsequent improvements in sample complexity by \cite{kaufmann2020adaptiverewardfreeexploration,menard_fast_2021,pmlr-v139-zhang21e}. Similar findings have been developed for settings involving function approximation \citep{wang2020rewardfreereinforcementlearninglinear,wagenmaker2022rewardfreerlharderrewardaware,NEURIPS2020_87736972,chen2022statisticalefficiencyrewardfreeexploration,qiu2022rewardfreerlkernelneural}. While existing reward-free methods efficiently estimate the transition kernel, they do not address the challenge of uniformly controlling high-dimensional transition estimation errors that arises in transfer settings. Our work introduces new tools to tackle this challenge in a sample-efficient manner.

\paragraph{Notation.}
We denote by $[n]$ the set $\{1,\cdots,n\}$ for any positive integer $n$, and use $\indic{\cdot}$ to represent the indicator function. For a function $f$ defined on $S$, we define its expectation under the probability measure $p$ as $ pf \triangleq \mathbb{E}_{s \sim p} f(s) $ and its variance as $\mathrm{Var}_p(f) \triangleq \mathbb{E}_{s \sim p} (f(s) - \mathbb{E}_{s' \sim p} f(s'))^2 = p(f - pf)^2 $. The total variation distance between probability measures $p$ and $q$ is defined as $\dtv{p}{q}\triangleq\sup_{A\subseteq S}|p(A)-q(A)|$. We use standard $O(\cdot)$ and $\Omega(\cdot)$ notation, where $f=O(g)$ means there exists some constant $C>0$ such that $f\le Cg$ (similarly for $\Omega(\cdot)$), and use the tilde notation $\widetilde{O}(\cdot)$ to suppress additional log factors. We denote the cardinality of a set $\mathcal{X}$ by $|\mathcal{X}|$. 
\section{Hybrid Transfer RL}
We begin by introducing the mathematical formulation of HTRL, benchmarking with standard online RL.
\begin{figure}[h]
    \centering
    \vspace{-0.4cm}
    \includegraphics[width=0.6\linewidth]{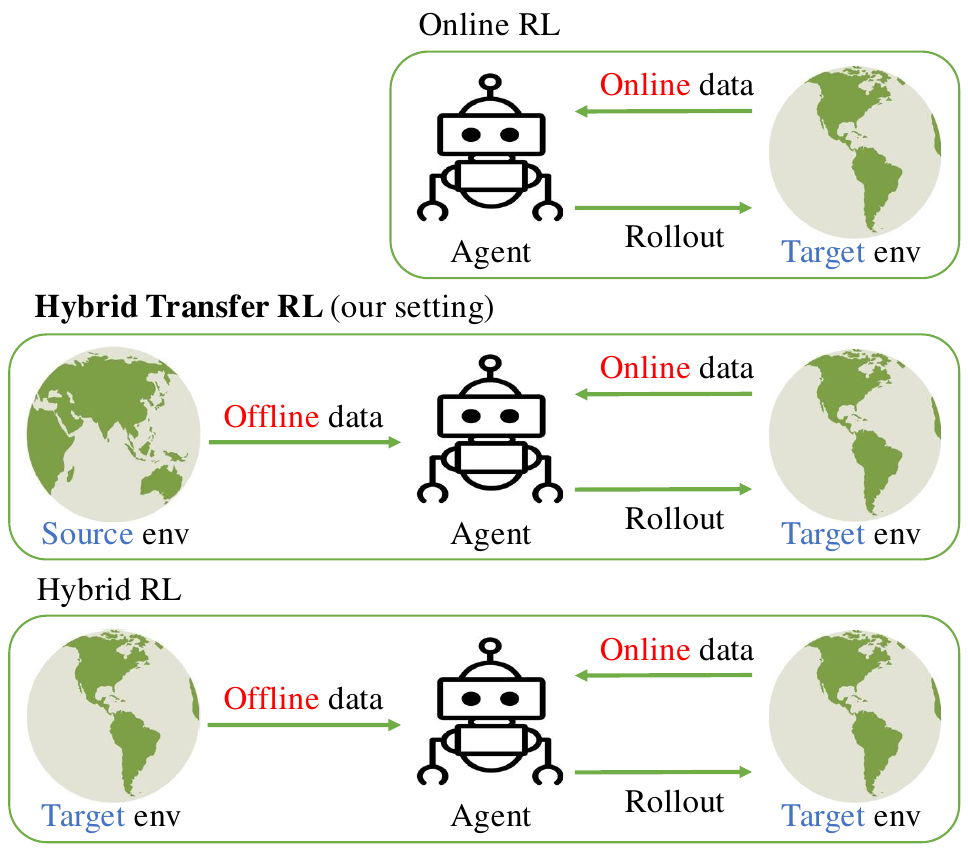}
    \caption{Comparison between different RL settings}
    \vspace{-0.4cm}
    \label{fig:settings}
\end{figure}
\paragraph{Background: Markov decision process (MDP).} We consider episodic Markov Decision Process $\mathcal{M}=(\mathcal{S},\mathcal{A},H,p,r,\rho)$, where $\mathcal{S}$ is the state space with size $S$, $\mathcal{A}$ is the action space with size $A$, $H$ is the horizon length. $p(\cdot\mid s,a): \mathcal{S}\times\mathcal{A}\mapsto \Delta(\mathcal{S})$ denotes the time-independent transition probability at each step, and the reward function is deterministic\footnote{For simplicity, we consider deterministic rewards, as estimating rewards is not a significant challenge in deriving sample complexity results.}, given by $r: \mathcal{S}\times\mathcal{A}\mapsto [0,1]$. In this setting, a Markovian policy is given by $\pi:=\{\pi_h\}_{h=1}^H$ where $\pi_h:\mathcal{S}\mapsto\Delta(\mathcal{A})$. Additionally, we assume that each episode of the MDP starts from an initial state generated from an unknown initial state distribution $\rho\in\Delta(\mathcal{S})$.

For a given transition $p$, the value function for state $s$ at step $h$ is defined as the expected cumulative future reward by executing policy $\pi$, which is given by $V_h^{p,\pi}(s):=\mathbb{E}_{p,\pi}[\sum_{i=h}^Hr(s_i,a_i)\mid s_h=s]$. Similarly, the state-action value function, or Q-function, is defined as $Q_h^{p,\pi}(s,a)=\mathbb{E}_{p,\pi}[\sum_{i=h}^H r(s_i,a_i)\mid s_h=s,a_h=a]$. We denote the weighted value function of policy $\pi$ by:
\begin{align*}
    V_1^{p,\pi}(\rho)=\mathbb{E}_{s\sim\rho}[V_1^{p,\pi}(s)].
\end{align*}
As is well known, there exists at least one deterministic policy that maximizes the value function and the Q-function simultaneously for all $(s,a,h)\in\mathcal{S}\times\mathcal{A}\times[H]$ \citep{bertsekas2010dynamic}. Let $\pi^\star$ denote an optimal deterministic policy, and the corresponding optimal value function $V_h^\star$ and optimal Q-function $Q_h^\star$ are defined respectively by $V_{h}^{p,\star} \triangleq V_h^{p,\pi^{\star}},\,\,Q_{h}^{p,\star} \triangleq Q_h^{p,\pi^{\star}},\,\forall (s,a,h)\in\mathcal{S}\times\mathcal{A}\times[H]$.
\subsection{Hybrid Transfer RL}
In HTRL, the agent can directly interact with the target MDP $\mathcal{M}_{\mathrm{tar}}=(\mathcal{S},\mathcal{A},H,p_{\mathrm{tar}},r,\rho)$ in episodes of length $H$. In an episode, at each step $h\in[H]$, the agent observes a state $s_h\in\mathcal{S}$, takes an action $a_h\in[H]$, receives a reward $r(s_h,a_h)$ and transitions to a new state $s_{h+1}$ according to the underlying transition probability $p_{\mathrm{tar}}(\cdot\mid s_h,a_h)$. 

Additionally, the agent has access to an offline dataset $\mathcal{D}_{\mathrm{src}}=\{(s_i,a_i,r_i,s_i')\}$ pre-collected from a source MDP $\mathcal{M}_{\mathrm{src}}=(\mathcal{S},\mathcal{A},H,p_{\mathrm{src}},r,\rho)$. The target and source MDPs share the same structure except for the transition probabilities (i.e. $p_{\mathrm{tar}}\neq p_{\mathrm{src}}$). For simplicity, we assume the reward signals in $\mathcal{M}_{\mathrm{src}}$ and $\mathcal{M}_{\mathrm{tar}}$ are the same; however, our analysis still holds when the reward signals differ. We assume $p_{\mathrm{src}}$ and $p_{\mathrm{tar}}$ are both unknown to the agent.
\paragraph{Goal.} With access to both $\mathcal{D}_{\mathrm{src}}$ and $\mathcal{M}_{\mathrm{tar}}$, the goal in HTRL is to find an $\varepsilon$-optimal policy for $\mathcal{M}_{\mathrm{tar}}$.  Specifically, the agent learns to find a policy $\hat{\pi}$ for $\mathcal{M}_{\mathrm{tar}}$, which satisfies that:
\begin{align*}
    V_1^{p_{\mathrm{tar}},\star}(\rho)-V_1^{p_{\mathrm{tar}},\hat{\pi}}(\rho)\le \varepsilon.
\end{align*}
\paragraph{Benchmarking with standard online RL.} The  baseline for HTRL is online RL, where the agent learns from scratch by directly interacting with $\mathcal{M}_{\mathrm{tar}}$ in episodes of length $H$. Different from HTRL, in online RL, the agent collects samples from $\mathcal{M}_{\mathrm{tar}}$ without any additional information, learning $\varepsilon$-optimal policy for $\mathcal{M}_{\mathrm{tar}}$. Offline RL cannot be a baseline because the $\mathcal{M}_{\mathrm{src}}$ differs from $\mathcal{M}_{\mathrm{tar}}$ and in general it is impossible to run an offline algorithm on $\mathcal{D}_{\mathrm{src}}$ to obtain an $\varepsilon$-optimal policy for $\mathcal{M}_{\mathrm{tar}}$.

Compared to online RL, the introduction of additional access to $\mathcal{D}_{\mathrm{src}}$ in HTRL naturally raises the question: can we achieve better sample efficiency by leveraging $\mathcal{D}_{\mathrm{src}}$? Unfortunately, the answer is negative, which will be highlighted in the next section.

\section{Minimax Lower Bound For HTRL}
In this section, we establish a minimax lower bound on the sample complexity for general HTRL, formally demonstrating that sample complexity improvements for general HTRL are not feasible.

Specifically, when $p_{\mathrm{tar}}$ is close to $p_{\mathrm{src}}$, one might expect that fewer samples from $\mathcal{M}_{\mathrm{tar}}$ are needed to reach a given performance level by leveraging additional information about $\mathcal{M}_{\mathrm{src}}$ because by the Simulation Lemma, we can already obtain a good initial policy from $\mathcal{M}_{\mathrm{src}}$. However, we show in \cref{theorem:lower-bound} 
that, in the worst case, the same number of samples from $\mathcal{M}_{\mathrm{tar}}$ is still required compared with pure online RL. The proof can be found in \cref{appendix:lower-bound}.

\begin{theorem}[Minimax lower bound for HTRL]
  \label{theorem:lower-bound}
  Given an optimality gap $\varepsilon$, consider for any $\mathcal{M}_{\mathrm{src}}$ the following set of possible MDPs:
    \begin{align*}
        \mathcal{M}_{\alpha} \triangleq \{ &\mathcal{M} = (\mathcal{S}, \mathcal{A}, H, p, r, \rho)\mid \\
        & \max_{(s, a) \in \sa}
        \dtv{p(\cdot \mid s, a)}{p_{\mathrm{src}}(\cdot \mid s, a)} \leq \alpha \},
    \end{align*}
  where $48\varepsilon/H^2\le \alpha\le 1$.  Suppose $S\ge 3$, $H\ge 3$, $A\ge 2$, $\varepsilon\le 1/48$. For any algorithm, there always exists a $\mathcal{M}_{\mathrm{src}}$ and a target MDP $\mathcal{M}_{\mathrm{tar}}\in \mathcal{M}_{\alpha}$, if the number of samples  $n$ collected from the target MDP is $O(H^3SA / \varepsilon^2)$,
  then the algorithm suffers from an $\varepsilon$-suboptimality gap:
  \begin{align*}
    \mathbb{E}_{\mathrm{tar}}\left[ V_{1}^{p_{\mathrm{tar}},\star}(\rho) - V_{1}^{p_{\mathrm{tar}},\hat{\pi}}(\rho)\right] \ge \varepsilon,
  \end{align*}
where $\mathbb{E}_{\mathrm{tar}}$ denotes the expectation with respect to the randomness during algorithm execution in the target MDP $\mathcal{M}_{\mathrm{tar}}$.
\end{theorem}

\cref{theorem:lower-bound} shows that the lower bound of sample complexity of general HTRL is $\Omega(H^3SA/\varepsilon^2)$, which matches the state-of-the-art sample complexity of pure online RL, $\widetilde{O}(H^3SA/\varepsilon^2)$ (e.g., \cite{menard_fast_2021}\footnote{\cite{menard_fast_2021} considers time-dependent transitions and the sample complexity result is $\widetilde{O}(H^4SA/\varepsilon^2)$, which in our setting translates into $\widetilde{O}(H^3SA/\varepsilon^2)$ due to time-independent transitions.}; \cite{wainwright2019variancereducedqlearningminimaxoptimal}). This demonstrates that, in general, practical transfer algorithms leveraging source environment data cannot reduce the sample complexity in the target environment. No matter what algorithms are used, there always exists a worst case where transfer learning cannot achieve better sample efficiency in the target environment, compared to pure online RL. However, this lower bound is conservative, motivating us to explore practically meaningful and feasible settings to derive problem-dependent sample complexity bounds.
\paragraph{Comparisons to prior lower bounds.}
To the best of our knowledge, this is the first lower bound on the sample complexity when leveraging information from a source environment to explore a new target environment with an unknown dynamics shift. We highlight the novelty of our lower bound result by comparing it with prior lower bounds:

\textit{Lower bounds for transfer learning in RL}: \cite{odonoghue2022variationalbayesianreinforcementlearning,ye2023powerpretraininggeneralizationrl,mutti2024testtimeregretminimizationmeta} provide lower bounds on regret in settings where the agent is trained on $N$ source tasks and is fine-tuned to the target task during testing. However, these lower bounds cannot be adapted to our setting as they assume the target task is one of the source tasks -- which is stronger than ours.

\textit{Lower bounds for pure online RL}:  The existing lower bound on the sample complexity of pure online RL is also $\Omega(H^3SA/\varepsilon^2)$ \citep{pac-azar}. This demonstrates that the improvement in the sample complexity lower bound from the introducing additional information from a source environment is at most a constant factor. The construction of the lower bound for HTRL follows a procedure similar to existing lower bounds for online RL \citep{lattimore2012pacboundsdiscountedmdps,yin2020near}, offline RL \citep{rashidinejad2023bridgingofflinereinforcementlearning} and hybrid RL without dynamics shift \citep{xie_policy_2022}. However, the new technical challenge in our setting is to bound the maximum information gain of the data from $\mathcal{M}_{\mathrm{src}}$. We address this difficulty using a proper change-of-measure approach. See \cref{appendix:lower-bound} for a detailed comparison.

\section{HTRL with Separable Shift}
Although improved sample efficiency is not achievable for general HTRL in the worst case, practical tasks are typically more manageable than these difficult instances. Inspired by practical tasks such as hierarchical RL \citep{chua_provable_2023} and meta RL \citep{chen2022understandingdomainrandomizationsimtoreal}, we instead focus on a class of HTRL with separable shift in the following. This setting allows us to leverage prior information about the degree of dynamic shift between the source and target environments. We then propose an algorithm, called HySRL, which achieves provably superior sample complexity compared to pure online RL.

\subsection{$\beta$-separable shfits}
We first introduce the definition of separable shift, characterized by the minimal degree of the dynamics shift between the source and target environments.

\begin{definition}[$\beta$-separable shift]\label{definition:separation}
Consider a target MDP $\mathcal{M}_{\mathrm{tar}}=(\mathcal{S},\mathcal{A},H,p_{\mathrm{tar}},r,\rho)$ and a source MDP $\mathcal{M}_{\mathrm{src}}=(\mathcal{S},\mathcal{A},H,p_{\mathrm{src}},r,\rho)$.
    The shift between $\mathcal{M}_{\mathrm{tar}}$ and $\mathcal{M}_{\mathrm{src}}$ is $\beta$-separable if for some $\beta \in (0,1]$, we have for all $(s, a) \in \mathcal{S}\times\mathcal{A}$,
    \begin{align*}
       &p_{\mathrm{src}}(\cdot \mid s, a) \neq p_{\mathrm{tar}}(\cdot \mid s, a)\\
        \implies &\dtv{p_{\mathrm{src}}(\cdot \mid s, a)}{p_{\mathrm{tar}}(\cdot \mid s, a)} \ge \beta.
    \end{align*}
\end{definition}

In other words, for any state-action pair $(s,a)$, the transitions in $\mathcal{M}_{\mathrm{src}}$ and $\mathcal{M}_{\mathrm{tar}}$ are either identical or different by at least the degree of $\beta$ w.r.t the TV distance metric. This definition is widely used to characterize the "distance" between tasks in hierarchical RL \citep{chua_provable_2023}, RL for latent MDPs \citep{kwon2021rllatentmdpsregret}, multi-task RL \cite{brunskill2013}, and meta RL \citep{mutti2024testtimeregretminimizationmeta,chen2022understandingdomainrandomizationsimtoreal}, serving the purpose of distinguishing different tasks with finite samples.

Such a minimal degree of dynamic shift, $\beta$, can often be estimated beforehand as prior information for specific problems in practice \citep{brunskill2013}. 
Therefore, in this section, we design algorithms under the assumption that $p_{\mathrm{tar}}$ and $p_{\mathrm{src}}$ are $\beta$-separable.

\begin{remark}[Separable shift makes HTRL feasible]
The lower bound in \cref{theorem:lower-bound} arises from potential challenging target MDPs that subtly differ from the source MDP and are specified based on the optimality gap $\varepsilon$. This subtlety requires extensive data to distinguish between them. However, in practice, the dynamic shift between source and target environments is usually independent of $\varepsilon$. By focusing on HTRL with a $\beta$-separable shift, where $\beta$ is independent of $\varepsilon$, we exclude over-conservative instances that are rare in practice.
\end{remark}

In addition to the aforementioned key definition --- $\beta$-separable dynamic shifts, we introduce another assumption for the reachability of the target MDP. Note that it is not tailored for our Hybrid Transfer RL setting, but widely adopted in extensive RL tasks such as standard RL, meta RL and multi-task RL \citep{JMLR:v11:jaksch10a,chen2022understandingdomainrandomizationsimtoreal,brunskill2013}, to ensure the agent's access to the entire environment (over all state-action pairs).

\begin{assumption}[$\sigma$-reachability]\label{definition:reachability}
    We assume the target MDP $\mathcal{M}_{\mathrm{tar}}$ has $\sigma$-reachability if there exists a constant $\sigma\in(0,1]$ so that 
    \begin{align*}
                \max_{\pi}\max_{h\in[H]} p_h^{\pi}(s,a)\ge \sigma, \quad \forall (s,a)\in\sa,
    \end{align*}
        where $p_h^{\pi}(s,a)$ is the probability of reaching $(s,a)$ at step $h$ by executing policy $\pi$ in $\mathcal{M}_{\mathrm{tar}}$.
\end{assumption} 

\subsection{Algorithm design: HySRL}
Focusing on HTRL with $\beta$-separable shift, now we are ready to introduce our algorithm HySRL, outlined in \cref{alg:hybrid}.
To explicitly characterize the set of state-action pairs where $p_{\mathrm{src}}$ and $p_{\mathrm{tar}}$ differ, we introduce the following definition. 
\begin{definition}[Shifted region]
    We define the shifted region $\mathcal{B}$ as the set of state-action pairs where the transitions in $\mathcal{M}_{\mathrm{src}}$ and $\mathcal{M}_{\mathrm{tar}}$ differ:
    \begin{equation*}
        \mathcal{B}\triangleq\{(s,a)\in \mathcal{S}\times\mathcal{A}\mid p_{\mathrm{src}}(\cdot \mid s,a)\neq p_{\mathrm{tar}}(\cdot \mid s,a) \}.
    \end{equation*}
\end{definition}

Although $p_{\mathrm{tar}}$ is unknown in advance, it is possible to invest a small number of online samples to estimate $p_{\mathrm{tar}}$ and identify the shifted region $\mathcal{B}$. This helps determine which part of $\mathcal{D}_{\mathrm{src}}$ can improve sample efficiency in $\mathcal{M}_{\mathrm{tar}}$, allowing us to focus further exploration on the remaining areas to learn an effective policy. Since, in many practical applications, the dynamcis shift typically affects only a small portion of the state-action space \citep{chua_provable_2023}, this approach can enable more sample-efficient exploration in $\mathcal{M}_{\mathrm{tar}}$. This intuition drives the design of \cref{alg:hybrid}.
\paragraph{\cref{alg:hybrid}: Hybrid separable-transfer RL (HySRL).}
At a high level, given a desired optimality gap $\varepsilon$, if $\sigma\beta$ is too small -- implying that an excessive number of samples is required to identify the shifted region -- \cref{alg:hybrid} chooses to ignore the offline dataset and instead relies on pure online learning. Otherwise, we proceed as follows: first, we run \cref{alg:rf} to obtain an estimated shifted region $\hat{\mathcal{B}}$, which, with high probability, matches the true shifted region $\mathcal{B}$. Next, we use the offline dataset $\mathcal{D}_{\mathrm{src}}$ and online data to design exploration bonuses and execute \cref{alg:ucbvi} to efficiently explore $\hat{\mathcal{B}}$, ultimately outputting a final policy for $\mathcal{M}_{\mathrm{tar}}$. Below, we outline the key steps of \cref{alg:hybrid}.
\begin{algorithm}[ht]
    \caption{Hybrid Separable-transfer RL (HySRL)}\label{alg:hybrid}
    \begin{algorithmic}[1]
        \REQUIRE{
        Parameters $\beta$, $\delta$, $\sigma$, $\varepsilon$, source dataset $\mathcal{D}_{\mathrm{src}}$
        }
        \IF{$\sigma\beta\le\sqrt{S/H}\varepsilon$}{
            \STATE $\hat{\mathcal{B}}\leftarrow\sa$ \textcolor{blue}{ // Abandon $\mathcal{D}_{\mathrm{src}}$}
        }
        \ELSE{
            \STATE $\hat{\mathcal{B}}\leftarrow$ call \cref{alg:rf}\textcolor{blue}{ // Estimate $\mathcal{B}$}
        }
        \ENDIF
        \STATE $\pi^{\mathrm{final}}\leftarrow$ call \cref{alg:ucbvi} with $\hat{\mathcal{B}}$\textcolor{blue}{ // Explore $\hat{\mathcal{B}}$}
        \RETURN $\pi^{\mathrm{final}}$
    \end{algorithmic}
\end{algorithm}
\paragraph{Step 1: Reward-free shift identification (\cref{alg:rf}).} Even with knowledge of $\beta$ and $\sigma$, accurately estimating $p_{\mathrm{tar}}$ to identify the shifted region $\mathcal{B}$ is still challenging, as we need to control the errors in estimating high-dimensional transitions with finite samples. 
To the best of our knowledge, no existing works have addressed this issue. 
\begin{algorithm}[ht]
    \caption{Reward-free shift identification}\label{alg:rf}
    \begin{algorithmic}[1]
        \REQUIRE Parameters $\beta$, $\delta$, $\sigma$, $\hat{p}_{\mathrm{src}}$
        \FOR{$t=0,1,2,\cdots$}{
            \FOR{$h=H,\cdots,1$}{
            \STATE Update $W_h^t$ using \cref{eq:definition-of-W}
            \STATE Update $\pi_h^{t+1}(\cdot)=\arg\max_{a\in\mathcal{A}}W_h^t(\cdot,a)$
            }
            \ENDFOR
            \STATE \textbf{Break} if $3\sqrt{\rho\pi_1^{t+1} W_1^t}+\rho\pi_1^{t+1} W_1^t\le\sigma\beta/8$
            \STATE Rollout $\pi^{t+1}$ and observe new online samples
            \FOR{$(s,a)\in\sa$}{
                \STATE Update $n^t(s,a)$, $n^t(s,a,s')$ and $\hat{p}^t_{\mathrm{tar}}(\cdot\mid s,a)$
            }
            \ENDFOR
        }
        \ENDFOR
        \RETURN Estimated shifted area $\hat{\mathcal{B}}$ by \cref{eq:estimated-shifted-area}
    \end{algorithmic}
\end{algorithm}

To this end, sufficient online data coverage is required for each $(s,a)$, which aligns with the motivation behind reward-free exploration to collect enough data and achieve optimality for any reward signal $r:\sa\mapsto[0,1]$. Inspired by RF-Express from \cite{menard_fast_2021}, we propose an algorithm in \cref{alg:rf}.  Specifically, we first define an uncertainty function $W_h^{t}(s,a)$, which characterizes data sufficiency in the $t^{\mathrm{th}}$ episode, recursively (with $W_{H+1}^t(s,a)=0$) for all $h\in[H]$ and $(s,a)\in\sa$,
\begin{align}
    W_h^t(s,a)&\triangleq\min\bigg(1,\frac{4Hg_1(\nt,\delta)}{\nt}\nonumber\\
    &+\sum_{s'}\hat{p}_{\mathrm{tar}}^t(s'\mid s,a)\max_{a'\in\mathcal{A}}W_{h+1}^t(s',a')\bigg),\label{eq:definition-of-W}
\end{align}
where $g_1(n,\delta)\triangleq\log(6SAH/\delta)+S\log(8e(n+1))$, $\nt\triangleq \sum_{\tau=1}^t\sum_{h=1}^H\indic{(s^\tau_h,a^\tau_h)=(s,a)}$ denotes the visitation count for $(s,a)$ in the first $t$ episodes and $\hat{p}_{\mathrm{tar}}^t(s,a)$ denotes the corresponding empirical transitions. Accordingly, we select $\pi_h^{t+1}(\cdot)=\arg\max_{a\in\mathcal{A}}W_h^t(\cdot,a)$ to collect online data from $\mathcal{M}_{\mathrm{tar}}$, update $\nt$, $\hat{p}_{\mathrm{tar}}^t(s,a)$ and $W_h^t(s,a)$, and stop until:
\begin{equation*}
    3\sqrt{\rho\pi_1^{t+1} W_1^t}+\rho\pi_1^{t+1} W_1^t\le\sigma\beta/8,
\end{equation*}
where $\rho\pi_1^{t+1} W_1^t=\sum_{s}\rho(s)W_1^t(s,\pi_1^{t+1}(s))$. This stopping criterion is designed to ensure that sufficient data coverage is achieved when \cref{alg:rf} stops. Beyond reward-free exploration, our design further guarantees the confidence intervals 
$$\dtv{p_{\mathrm{tar}}(\cdot\mid s,a)}{\hat{p}^t_{\mathrm{tar}}(\cdot\mid s,a)}\le \beta/4$$ 
is constructed for each $(s,a)$. Then, we estimated the shifted region as:
\begin{align}
    \hat{\mathcal{B}}&\triangleq \left\{(s,a)\in\sa\mid\right. \nonumber \\
    &\quad\left.\dtv{\hat{p}_{\mathrm{src}}(\cdot \mid s, a)}{\hat{p}^t_{\mathrm{tar}}(\cdot \mid s, a)}> \beta/2\right\},\label{eq:estimated-shifted-area}
\end{align}
where $\hat{p}^t_{\mathrm{src}}$ denotes the empirical transitions in $\mathcal{D}_{\mathrm{src}}$. We show that by executing \cref{alg:rf}, the shifted region $\mathcal{B}$ can be identified with high probability within a sample size from $\mathcal{M}_{\mathrm{tar}}$ that is independent of $\varepsilon$, as formally stated in \cref{lem:sample-recognization}. The proof of \cref{lem:sample-recognization} can be found in \cref{appendix:recognization}. 

\begin{lemma}[Sample-efficient shift identification]\label{lem:sample-recognization}
    Let \cref{definition:reachability} hold, and $\delta \in(0,1)$ be given. 
    Suppose the shift between $\mathcal{M}_{\mathsf{tar}}$ and $\mathcal{M}_{\mathsf{src}}$ is $\beta$-separable, and $\mathcal{D}_{\mathrm{src}}$ contains at least $\widetilde{\Omega}(S/\beta^2)$ samples for $\forall (s,a)\in\sa$.
      With probability at least $1-\delta/2$, \cref{alg:rf} can output an estimate of $p_{\mathrm{tar}}$ satisfying
    \begin{equation*}
        \dtv{p_{\mathrm{tar}}(\cdot\mid s,a)}{\hat{p}^t_{\mathrm{tar}}(\cdot\mid s,a)}\le \beta/4, \,\forall(s,a)\in\sa,
    \end{equation*}
    along with the estimated shifted region 
    $\hat{\mathcal{B}}=\mathcal{B}$, using $\widetilde{O}({H^2S^2A}/{(\sigma\beta)^2})$
    samples collected from $\mathcal{M}_{\mathrm{tar}}$.
\end{lemma}
The confidence interval for transitions with finite-sample guarantees in \cref{lem:sample-recognization} is estabilished by extending reward-free exploration to accomodate more general reward functions $r:[H]\times\sa\times\mathcal{S}\mapsto[0,1]$ in the analysis. 
\paragraph{Step 2: Hybrid UCB value iteration (\cref{alg:ucbvi}).} Once we have the estimated shifted region $\hat{\mathcal{B}}$, it is intuitive for the agent to focus more on exploring the estimated shifted region $\hat{\mathcal{B}}$. To achieve this, we introduce an algorithm that incorporates the additional source dataset $\mathcal{D}_{\mathrm{src}}$ in the design of the exploration bonus summarized in \cref{alg:ucbvi}.  

\begin{algorithm}[ht]
    \caption{Hybrid UCB Value Iteration}\label{alg:ucbvi}
    \begin{algorithmic}[1]
        \REQUIRE Parameters $\delta$, $\varepsilon$, $\hat{\mathcal{B}}$, $\mathcal{D}_{\mathrm{src}}$
        \FOR{$t=0,1,2,\cdots$}{
            \FOR{$h=H,\cdots,1$}{
            \STATE Update $\overline{Q}_h^t$, $G^t_h$ using \cref{eq:upper,eq:definition-G}
            \STATE Update $\pi_h^{t+1}(\cdot)=\arg\max_{a\in\mathcal{A}}\overline{Q}_h^t(\cdot,a)$
            }
            \ENDFOR
            \STATE \textbf{Break} if $\rho\pi_1^{t+1} G_1^t\le\varepsilon$
            \STATE Rollout $\pi^{t+1}$ and observe new online samples
            \FOR{$(s,a)\in\hat{\mathcal{B}}$}{ 
                \STATE Update $n^t(s,a)$, $n^t(s,a,s')$ and $\hat{p}^t_{\mathrm{tar}}(\cdot\mid s,a)$
                \STATE \textcolor{blue}{// Only update $n^t$ and $\hat{p}^t_{\mathrm{tar}}$ inside $\hat{\mathcal{B}}$}
            }
            \ENDFOR
        }
        \ENDFOR
        \RETURN $\pi^{\mathrm{final}}=\pi^{t+1}$
    \end{algorithmic}
\end{algorithm}

This algorithm is inspired by BPI-UCBVI in \cite{menard_fast_2021}; however, in our problem, we carefully design the exploration bonus to leverage the additional offline dataset $\mathcal{D}_{\mathrm{src}}$ while controlling potential bias that it introduces. To effectively use $\mathcal{D}_{\mathrm{src}}$ while avoiding potential bias, we define the upper confidence bounds of the optimal Q-functions and value functions for the estimated shifted region $\hat{\mathcal{B}}$ and its complement $\sa\,/\,\hat{\mathcal{B}}$, respectively:
\begin{subequations}
\begin{align}
    &\overline{Q}_h^t(s, a)\triangleq \min \bigg(3\sqrt{\tmvar(\overline{V}_{h+1}^t)(s,a) \frac{g_2(\tnt(s,a), \delta)}{\tnt(s,a)}} \nonumber\\
    &+ \frac{14H^2 g_1(\tnt(s,a), \delta)}{\tnt(s,a)} + \frac{1}{H} \tpt (\overline{V}_{h+1}^t - \underline{V}_{h+1}^t)(s,a)\nonumber\\
    &+ \tpt \overline{V}_{h+1}^t(s,a)+r(s,a),H\bigg),\label{eq:upper}\\
    &\overline{V}_h^t(s)\triangleq \max_{a \in \mathcal{A}}\,\overline{Q}_h^t(s, a),\ \overline{V}_{H+1}^t(s)\triangleq 0,
\end{align}
\end{subequations}
where $g_2(n,\delta)\triangleq\log(6SAH/\delta)+\log(8e(n+1))$, $\underline{V}_{h+1}^t$ is a lower bound of the optimal value defined similarly in \cref{appendix:main-result}, and $\tmvar(\cdot)$ denotes the empirical variance under $\tpt$. Here, for $(s,a)\in\hat{\mathcal{B}}$, we have $\tnt(s,a)\triangleq n^t(s,a)$ and $\tpt(\cdot\mid s,a)\triangleq \hat{p}_{\mathrm{tar}}^t(\cdot\mid s,a)$; for $(s,a)\notin\hat{\mathcal{B}}$, we have $\tnt(s,a)\triangleq n_{\mathrm{src}}(s,a)$ and $\tpt(\cdot\mid s,a)\triangleq \hat{p}_{\mathrm{src}}(\cdot\mid s,a)$, where $n_{\mathrm{src}}$ denotes the visitation count in $\mathcal{M}_{\mathrm{src}}$. Note that we only update the visitation count in $\hat{\mathcal{B}}$ to remove statistical dependency.

Aiming to achieve optimality in $\mathcal{M}_{\mathrm{tar}}$, we choose $\pi_h^{t+1}(\cdot)=\arg\max_{a\in\mathcal{A}}\overline{Q}_h^t(\cdot,a)$ to collect samples from $\mathcal{M}_{\mathrm{tar}}$ in \cref{alg:rf}. Accordingly, we define the following function $G_h^t(s,a)$ to serve as an upper bound on the optimality gap $V_h^{p_{\mathrm{tar}},\star}-V_h^{p_{\mathrm{tar}},\pi^{t+1}}$ (with $G_{H+1}(s,a)=0$):
\begin{small}
\begin{equation}\label{eq:definition-G}
\begin{aligned}
    &G_h^t(s, a) \triangleq \min \bigg(H,  6\sqrt{\tmvar(\overline{V}_{h+1}^t)(s,a)\frac{g_2(\tnt(s,a), \delta)}{\tnt(s,a)}} \\
    &+  \frac{35H^2 g_1(\tnt(s,a), \delta)}{\tnt(s,a)}
     + ( 1 + \frac{3}{H}) \tpt \pi_{h+1}^{t+1} G_{h+1}^t(s,a)\bigg),
\end{aligned}
\end{equation}
\end{small}%
\cref{alg:ucbvi}  stops when $\rho\pi_1^{t+1} G_1^t\le \varepsilon$, indicating that $\varepsilon$-optimality is achieved in the target domain $\mathcal{M}_{\mathrm{tar}}$. This procedure requires at most $\widetilde{O}(H^3|\mathcal{B}|/\varepsilon^2)$ samples from $\mathcal{M}_{\mathrm{tar}}$, as detailed in the final results in the next section.

\subsection{Theoretical guarantees: sample complexity}
In this subsection, we discuss the total sample complexity of \cref{alg:hybrid}, highlighting its sample efficiency gains compared to the state-of-the-art pure online RL sample complexity and connections to practical transfer algorithms.
\begin{theorem}[Problem-dependent sample complexity]\label{theorem:sample-complexity}
    Let \cref{definition:reachability} hold,  and $\delta \in(0,1)$ and $\varepsilon\in(0,1]$ be given. Suppose the shift between $\mathcal{M}_{\mathsf{tar}}$ and $\mathcal{M}_{\mathsf{src}}$ is $\beta$-separable, and $\mathcal{D}_{\mathrm{src}}$ contains at least $\widetilde{\Omega}(H^3/\varepsilon^2+S/\beta^2)$ samples for $\forall (s,a)\in\sa$. With probability at least $1-\delta$, the output  policy $\pi^{\mathrm{final}}$ of \cref{alg:hybrid} satisfies
\begin{equation}
    V_1^{p_{\mathrm{tar}},\star}(\rho)-V_1^{p_{\mathrm{tar}},\pi^{\mathrm{final}}}(\rho)\le\varepsilon,
\end{equation}
if the total number of online samples collected from $\mathcal{M}_{\mathrm{tar}}$ is  
\begin{equation}\label{eq:hybrid-sample}
    \widetilde{O}\left(\min\left(\frac{H^3SA}{\varepsilon^2},\frac{H^3|\mathcal{B}|}{\varepsilon^2}+\frac{H^2S^2A}{(\sigma\beta)^2}\right)\right).
\end{equation}
\end{theorem}
\cref{theorem:sample-complexity} provides a problem-dependent sample complexity of \cref{alg:hybrid} that is at least as good as the state-of-the-art $\widetilde{O}(H^3SA/\varepsilon^2)$ in pure online RL \citep{menard_fast_2021,wainwright2019variancereducedqlearningminimaxoptimal}. Specifically, for a given $\beta$:
\begin{itemize}[nosep,leftmargin=*]
    \item When $\varepsilon\ge\Omega(\sqrt{H/S}\sigma\beta)$: it captures the scenarios where the desired optimality gap $\varepsilon$ is at least the order of the degree of the dynamics shift $\beta$. In this case, the sample complexity of \cref{alg:hybrid} becomes $\widetilde{O}(H^3SA/\varepsilon^2)$, which matches the state-of-the-art pure online RL sample complexity, showing that our framework provably avoids negative transfer in terms of sample efficiency.
    \item When $\varepsilon<\Omega(\sqrt{H/S}\sigma\beta)$: the comparisons between the sample complexity of \cref{alg:hybrid} in \cref{eq:hybrid-sample} and the state-of-the-art pure online RL is as follows:
    \begin{equation*}
      \widetilde{O}\left(\frac{H^3|\mathcal{B}|}{\varepsilon^2}\right)  \quad v.s. \quad  \widetilde{O}\left(\frac{H^3SA}{\varepsilon^2}\right),
    \end{equation*}
    where $|\mathcal{B}|$ represents the cardinality of the shifted region, a problem-dependent parameter in HTRL that is strictly no larger than $SA$. It indicates that \cref{alg:hybrid} provably achieves better sample efficiency than state-of-the-art pure online RL algorithms in HTRL tasks, as long as the shift does not cover the entire state-action space, as validated in \cref{section:experiments}. In many practical scenarios, such as training cooking agents \citep{beck_survey_2023} or autonomous driving \citep{xiong2016combiningdeepreinforcementlearning}, environmental variations between source and target environments (e.g., different kitchen layouts or obstacle positions) typically affect only a small portion of the state-action space, meaning $|\mathcal{B}|\ll SA$, with a large separable shift. This enables significant sample efficiency gains from reusing the source dataset.
\end{itemize}

Our results demonstrate that for HTRL tasks with $\beta$-separable shift between source and target environments, \cref{alg:hybrid} provably avoids harmful information transfer and enhances sample efficiency compared to pure online RL. 
While \cref{definition:separation} depends on $\beta$, we evaluate \cref{alg:hybrid} in broader scenarios where an inaccurate $\beta$ is used, as discussed in \cref{section:experiments}, demonstrating the robustness of \cref{alg:hybrid}. 

\paragraph{Connections with practical cross-domain transfer algorithms.} Practical algorithms for cross-domain transfer RL often involve training a neural network classifier to distinguish between source and target transitions \citep{eysenbach_off-dynamics_2021,liu2022daradynamicsawarerewardaugmentation,niu_h2o_2023,wen_contrastive_2024} and reusing source data accordingly. Our sample complexity results provide theoretical insights for determining the data collection budget in the target domain. They also demonstrate that the estimated transition shift serves as an effective metric for utilizing the source data and can provably improve sample efficiency.

\paragraph{Extensions of \cref{theorem:sample-complexity}: variants of source data.}
In \cref{theorem:sample-complexity}, we assume abundant samples from the source domain, which is a common assumption since we primarily focus on sample complexity in the target domain $\mathcal{M}_{\mathrm{tar}}$. However, even when the source dataset $\mathcal{D}_{\mathrm{src}}$ is insufficient, similar results hold. In particular, we consider the set of state-action pairs where $\mathcal{D}_{\mathrm{src}}$ lacks sufficient samples:
\begin{align*}
    \mathcal{C}\triangleq\{(s,a)\in\sa\mid 
    n_{\mathrm{src}}(s,a)<\widetilde{\Omega}(H^3/\varepsilon^2+S/\beta^2) \}.
\end{align*}
By adjusting the input of \cref{alg:ucbvi} to $\hat{\mathcal{B}}\cup \mathcal{C}$,
\cref{alg:hybrid} can still achieve the identical optimality with the sample complexity as below:
    \begin{equation*}
        \widetilde{O}\left(\min\left(\frac{H^3SA}{\varepsilon^2},\frac{H^3|\mathcal{B}\cup\mathcal{C}|}{\varepsilon^2}+\frac{H^2S^2A}{(\sigma\beta)^2}\right)\right).
    \end{equation*}
    
Similarly, when $N$ datasets from $N$ different source MDPs are available, \cref{alg:hybrid} can still function by executing \cref{alg:rf} once to identify the shifts in the target transition relative to each source transition and selecting useful source data accordingly. Let $\mathcal{B}_i$ denote the corresponding shifted region for each source MDP~$i$. Under the conditions of \cref{theorem:sample-complexity}, the required  sample complexity in this setting becomes
\begin{equation*}
    \widetilde{O}\left(\min\left(\frac{H^3SA}{\varepsilon^2},\frac{H^3|\cap_{i\in[N]}\mathcal{B}_i|}{\varepsilon^2}+\frac{H^2S^2A}{(\sigma\beta)^2}\right)\right).
\end{equation*}

\section{Experiments}\label{section:experiments}
We evaluate our proposed algorithm by comparing it to the state-of-the-art online RL baseline, BPI-UCBVI \cite{menard_fast_2021}, in the GridWorld environment  ($S=16,A=4,H=20$).

In the source and the target environments, the agent may fail to take an action and go to a wrong direction. Compared with the source environment, the target environment includes three absorbing states. The source dataset is collected by running \cref{alg:rf} in the source environment for $T=1\times 10^5$ episodes, which satisfies the conditions in \cref{theorem:sample-complexity}. We implement both algorithms in the benchmark rlberry \citep{rlberry}, similar to that in \cite{menard_fast_2021}. See \cref{appendix:experiment} for a detailed introduction of the experiment setup. The code is available at \url{https://github.com/SilentEchoes77/hybrid-transfer-rl}. The results are presented in \cref{fig:experiments}, averaged over 5 random seeds with a 95\% confidence interval. 

\begin{figure}[ht]
    \centering
    \subfloat[]{
        \centering
        \includegraphics[width=0.47\linewidth]{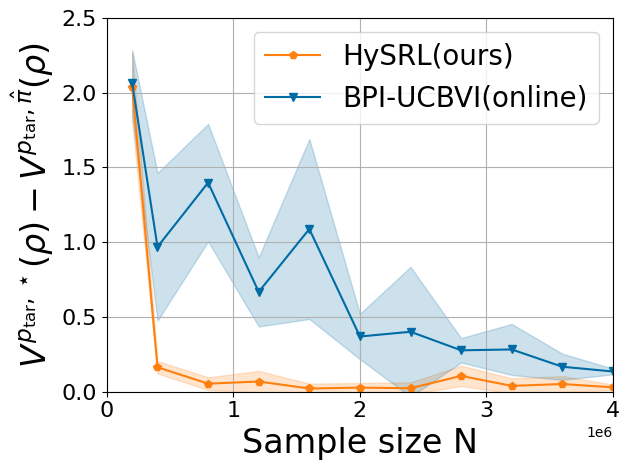}
        \label{fig:gap-ci}
    }
    \hspace{-0.4cm}
    \subfloat[]{
        \centering
        \includegraphics[width=0.47\linewidth]{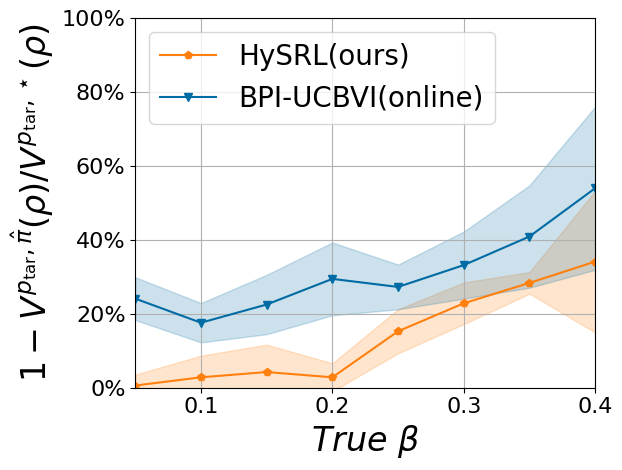}
        \label{fig:percentage-drop-ci}
    }
    \caption{\cref{fig:gap-ci} shows the optimality gap of HySRL (ours) and BPI-UCBVI as the sample size varies. \cref{fig:percentage-drop-ci} presents the percentage optimality gap of HySRL (ours) and BPI-UCBVI as the true $\beta$ varies.}
    \label{fig:experiments}
\end{figure}

As shown in \cref{fig:gap-ci}, \cref{alg:hybrid} learns the optimal policy with approximately $1\times 10^6$ samples from the target environment. In contrast, BPI-UCBVI converges more slowly, highlighting the data inefficiency of pure online RL. This demonstrates that transferring shifted-dynamics data from a source environment can significantly improve sample efficiency.

To assess whether a correct $\beta$ is necessary, we conduct an ablation study with the input $\beta=0.45$, while the true $\beta$ ranges from 0.05 to 0.4. As shown in \cref{fig:percentage-drop-ci}, even when \cref{definition:separation} is not satisfied, the performance degradation of the output policy from \cref{alg:hybrid} is minor and still outperforms BPI-UCBVI within finite samples, demonstrating the robustness of our algorithm.
\section{Conclusion}
This paper introduces Hybrid Transfer RL, providing a framework to analyze the finite-sample guarantees of practical transfer algorithms. We establish a worst-case lower bound for general HTRL, demonstrating that it cannot outperform online RL in its most general form. 
However, in more practical scenarios, we show that transferring shifted-dynamics data can provably reduce sample complexity in the target environment, offering theoretical insights for algorithm design.

\section*{Acknowledgment}
The work of C. Qu is supported in part by NSFC through 723B1001 and by the Summer Undergraduate Research Fellowships at California Institute of Technology. The work of L. Shi is supported in part by the Resnick Institute and Computing, Data, and Society Postdoctoral Fellowship at California Institute of Technology. K. Panaganti is supported in part by the Resnick Institute and the ‘PIMCO Postdoctoral Fellow in Data Science’ fellowship at the California Institute of Technology. The work of P. You is supported in part from NSFC through 723B1001, 72431001, 72201007, T2121002, 72131001. The work of A. Wierman is supported in part from the NSF through CNS-2146814, CPS-2136197, CNS-2106403, NGSDI-2105648.
\bibliographystyle{apalike}
\bibliography{hybrid}


\appendix
\section{Proof of the Minimax Lower Bound}\label{appendix:lower-bound}
\subsection{Preliminaries and Notations}
We consider the following transfer setting: for a given source MDP $\mathcal{M}_{\mathrm{src}}=(\mathcal{S},\mathcal{A},H,p_{\mathrm{src}},r,\rho)$, the target MDP $\mathcal{M}_{\mathrm{tar}}=(\mathcal{S},\mathcal{A},H,p_{\mathrm{tar}},r,\rho)$ is similar to the source MDP in the sense that:
\begin{equation*}
    \max_{s,a\in\sa}\dtv{p_{\mathrm{tar}}(\cdot \mid s, a)}{p_{\mathrm{src}}(\cdot \mid s, a)}\le \alpha.
\end{equation*}
We assume that the target MDP transition $p_{\mathrm{tar}}$ is unknown to us. Thus all the possible target MDPs compose the following set $\mathcal{M}_{\alpha}$, which is given by:
\begin{equation*}
    \mathcal{M}_{\alpha}\triangleq \left\{\mathcal{M}=(\mathcal{S},\mathcal{A},H,p,r,\rho) \mid \max_{s,a\in\sa}\dtv{p(\cdot \mid s, a)}{p_{\mathrm{src}}(\cdot \mid s, a)}\le \alpha\right\}.
\end{equation*}
We use $\mathbb{E}_{\mathcal{M}_{\mathrm{tar}}}$ to denote the expectation with respect to the randomness during algorithm execution in the target MDP $\mathcal{M}_{\mathrm{tar}}$. 

The proof of \cref{theorem:lower-bound} follows a similar construction as in \cite{xie_policy_2022,rashidinejad2023bridgingofflinereinforcementlearning,lattimore2012pacboundsdiscountedmdps}. The new technical challenge in our setting is to bound the information gain of the data from $\mathcal{M}_{\mathrm{src}}$ with different dynamics. Specifically,
\begin{itemize}
    \item Although the minimax lower bound result in \cite{rashidinejad2023bridgingofflinereinforcementlearning} considers different dynamics, they construct the MDP class based on the offline occupancy measures, which is quite different from our construction. We instead put more emphasis on the source and target environmental gaps in our construction.
    \item \cite{xie_policy_2022,rashidinejad2023bridgingofflinereinforcementlearning,lattimore2012pacboundsdiscountedmdps} do not assume access to an additional source offline dataset. In contrast, in our setting, we need to control the information gain and bias from this additional source dataset.
\end{itemize}
As an overview of the proof, we construct a set of hard instance MDPs as the source MDP and the possible target MDPs, and demonstrate that the maximum sample efficiency gain from $\mathcal{M}_{\mathrm{src}}$ can be bounded via a change-of-measure approach.
\subsection{Proof of \cref{theorem:lower-bound}}
\begin{proof}
\begin{figure}[htbp]
    \centering
    \includegraphics[width=0.5\linewidth]{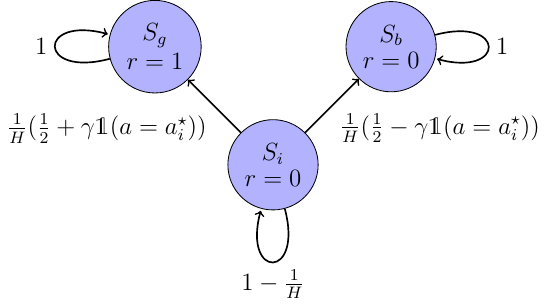}
    \caption{Hard MDPs}
    \label{fig:hard-instance}
\end{figure}
We construct a set of MDPs with $S+2$ states, $A$ actions and $H$ steps, which are replicas of the hard MDP presented in \cref{fig:hard-instance}. Without loss of generality, we assume $S\ge 1$ and $H\ge 3$. The action space is $[A]$. For the state space, each hard MDP contains $S$ bandit states denoted by $\{s_i\}_{i\in[S]}$, one good state $s_g$ and one bad state $s_b$. The initial state distribution is a uniform distribution over $\{s_i\}_{i\in[S]}$, given by $\rho(s_i)=1/S$, $\forall i \in[S]$.

For each hard MDP, at each bandit state $s_i$, the transition probability is given by:
\begin{equation*}
\left\{
\begin{aligned}
  & p(s_i \mid  s_i, a) = 1 - \frac{1}{H},\,\forall\,a\in [A], \\
  & p(s_g \mid  s_i, a) = \frac{1}{H}(\frac{1}{2}+\gamma\indic{a=a_i^\star}),\,\forall\,a\in[A], \\
  & p(s_b \mid  s_i, a) = \frac{1}{H}(\frac{1}{2} - \gamma\indic{a=a_i^\star}),\,\forall\,a\in[A],
\end{aligned}
\right.
\end{equation*}
where $\gamma\in[0,1/3]$ is a constant to be specified later, and $a_i^\star\in[A]$ is the unique optimal action at each bandit state $s_i$. We denote each hard MDP by $\mathcal{M}_{\ba^\star}$, where $\ba^\star\in[A]^S$ denotes the optimal action vector at each bandit state. Therefore, we have $[A]^S$ different hard MDPs with different optimal actions. The agent receives zero reward at the bandit states, that is, $r(s_i,a)=0$, $\forall a\in [A]$. 

Apart from the bandit states, $s_g$ and $s_b$ are both absorbing states with $p(s_g\mid s_g,a)=p(s_b\mid s_b,a)=1$, $\forall a\in [A]$. The agent receive reward $1$ at $s_g$ and zero reward at $s_b$. That is, $r(s_g,a)=1$, $r(s_b,a)=0$, $\forall a\in [A]$, $\forall i\in [S]$. 

Furthermore, we define a special MDP $\mathcal{M}_{\mathbf{0}}$ with no optimal actions. Specifically, the transition probability at bandit states in $\mathcal{M}_{\mathbf{0}}$ is given by:
\begin{equation*}
\left\{
\begin{aligned}
  & p(s_i \mid  s_i, a) = 1 - \frac{1}{H},\,\forall\,a\in [A], \\
  & p(s_g \mid  s_i, a) = \frac{1}{2H},\,\forall\,a\in[A], \\
  & p(s_b \mid  s_i, a) = \frac{1}{2H},\,\forall\,a\in[A],
\end{aligned}
\right.
\end{equation*}
while others remain the same as the hard MDPs. We choose this special MDP as our source MDP $\mathcal{M}_{\mathrm{src}}=\mathcal{M}_{\mathbf{0}}$. By the definition of $\mathcal{M}_{\ba^\star}$, we clearly see that when $\gamma\le \frac{48\varepsilon}{H}$, each hard instance $\mathcal{M}_{\ba^\star}$ belong to the possible target MDP set $\mathcal{M}_\alpha$ since:
\begin{align*}
    \max_{s,a\in\sa} \dtv{p_{\mathbf{0}}(\cdot\mid s,a)}{p_{\ba^\star}(\cdot\mid s,a)}=\frac{\gamma}{H}\le \frac{48\varepsilon}{H^2}\le \alpha.
\end{align*}
To prove \cref{theorem:lower-bound}, we only need to show that \cref{theorem:lower-bound} holds on these hard MDPs. That is, we restrict ourselves to the subcase where:
\begin{align*}
    \mathcal{M}_{\mathrm{tar}}\in \{\mathcal{M}_{\ba^\star}\}\subset \mathcal{M}_{\alpha}.
\end{align*} 
In the following steps, we condider the average suboptimality gap over $\{\mathcal{M}_{\ba^\star}\}$ and link this gap to the number of samples from the target MDP. The key step is to bound the information gain from the source MDP via a change-of-measure approach.

To begin with, we consider the uniform priori $\nu$ over $\ba^\star$(i.e., the hard MDPs), given by $\nu(\ba^\star)=1/A^S$, $\forall \ba^\star\in[A]^S$. For any algorithm ALG, we aim to link the following average suboptimality gap with the number of samples from the target MDP:
\begin{equation*}
    \mathbb{E}_{\ba^\star\sim\nu}\mathbb{E}_{\ba^\star}\left[  V_{1, M_{\ba^\star}}^\star - V_{1, M_{\ba^\star}}^{\hat{\pi}} \right],
\end{equation*}
where $\mathbb{E}_{\ba^\star}$ is taken with respect to the execution of ALG in $\{\mathcal{M}_{\ba^\star}\}$.
\paragraph{Decompose suboptimality gaps.}Let $\hat{\pi}(\cdot\mid \cdot)$ denotes the output policy of ALG. First, note that the state distribution $d^\pi_h(s_i)=1/S\cdot (1-1/H)^{h-1}\triangleq d_h(s_i)$ ($\forall (h,s)\in[H]\times[S]$) does not depend on the policy $\pi$. Therefore, we have
\begin{align*}
  & \quad V_{1, M_{\ba^\star}}^\star - V_{1, M_{\ba^\star}}^{\hat{\pi}} \\
  & = \sum_{h=1}^H \sum_{i=1}^S d^{\hat{\pi}}_h(s_i) \cdot \left[ \frac{1}{H}\left(\frac{1}{2}+\gamma\right) - \frac{1}{2H}\right] \cdot (1-\hat{\pi}_h(a^\star_{i}\mid s_i)) \cdot (H-h) \\
  & = \sum_{h=1}^H \sum_{i=1}^S \frac{1}{S}\left(1 - \frac{1}{H}\right)^{h-1}(1-h/H) \gamma \cdot (1-\hat{\pi}_h(a^\star_{i}\mid s_i)). 
\end{align*}
Taking expectation with respect to the algorithm execution within the MDP $M_{\ba^\star}$, we get
\begin{align*}
  \mathbb{E}_{\ba^\star}\left[  V_{1, M_{\ba^\star}}^\star - V_{1, M_{\ba^\star}}^{\hat{\pi}} \right] &= \sum_{h=1}^H \sum_{i=1}^S \frac{1}{S}\left(1 - \frac{1}{H}\right)^{h-1}(1-h/H)\gamma \cdot\mathbb{E}_{\ba^\star}(1-\hat{\pi}_h(a^\star_{i}\mid s_i))  \\
  & = \sum_{h=1}^H \sum_{i=1}^S \frac{1}{S}\left(1 - \frac{1}{H}\right)^{h-1}(1-h/H)\gamma \cdot \mathbb{E}_{\ba^\star}(1-\hat{\pi}_h(a^\star_{i}\mid s_i)) \\
  & \ge \frac{\gamma}{3S}\cdot \sum_{h=1}^H \sum_{i=1}^S(1-h/H) \mathbb{E}_{\ba^\star}(1-\hat{\pi}_h(a^\star_{i}\mid s_i)).\tag{$(1-1/H)^{h-1}\ge 1/e\ge 1/3$}
\end{align*}
This gives us:
\begin{equation*}
    \mathbb{E}_{\ba^\star}\left[  V_{1, M_{\ba^\star}}^\star - V_{1, M_{\ba^\star}}^{\hat{\pi}} \right]\ge\frac{\gamma}{3S}\cdot \sum_{h=1}^H \sum_{i=1}^S(1-h/H) \mathbb{E}_{\ba^\star}(1-\hat{\pi}_h(a^\star_{i}\mid s_i)).
\end{equation*}
\paragraph{Bound information gain.}Let $L_{h,i}(\hat{\pi},\ba)\triangleq \mathbb{E}_{\ba^\star}(1-\hat{\pi}_h(a^\star_{i}\mid s_i))$ denotes the loss function. Therefore, we have:
\begin{equation}\label{eq:sub-to-loss}
    \mathbb{E}_{\ba^\star\sim\nu}\mathbb{E}_{\ba^\star}\left[  V_{1, M_{\ba^\star}}^\star - V_{1, M_{\ba^\star}}^{\hat{\pi}} \right]\ge\frac{\gamma}{3S}\cdot \sum_{h=1}^H \sum_{i=1}^S(1-h/H)\mathbb{E}_{\ba^\star\sim\nu}L_{h,i}(\hat{\pi},\ba).
\end{equation}
Let $N_{\ba^\star}(s_i,a,s')$ denote the visitation of $(s_i,a,s')$ in $\mathcal{M}_{\ba^\star}$ and $N_{\mathbf{0}}(s_i,a,s')$ denote the visitation of $(s_i,a,s')$ in $\mathcal{M}_{\mathbf{0}}$. Since the only data that reveal information about $a_i^\star$ is $N_{\ba^\star}(s_i,\cdot,\cdot)$ and $N_{\mathbf{0}}(s_i,\cdot,\cdot)$, the visitation count at $s_i$, the posterior of $a_i^\star$ depends only on $N_{\ba^\star}(s_i,\cdot,\cdot)$ and $N_{\mathbf{0}}(s_i,\cdot,\cdot)$, the sufficient statistics for this posterior. Suppose the output policy $\hat{\pi}_h(a_i^\star,s_i)=f(N_{\mathbf{0}},N_{\ba^\star})$ is given by some measure function $f:(N_{\mathbf{0}},N_{\ba^\star})\mapsto[0,1]$. By \citet[Theorem 1.1 of Section 4]{lehmann2006theory}, we have:
\begin{align*}
  \mathbb{E}_{\ba^\star\sim \nu}\left[L_{h,i}(\hat{\pi}, \ba^\star)\right] \ge \inf_{f} \mathbb{E}_{\ba^\star\sim \nu}\mathbb{E}_{N_{\ba^\star}\sim (p_{\ba^\star}, \mathrm{ALG}), N_{\mathbf{0}}\sim (p_{\mathbf{0}}, \mathrm{ALG})}(1-f(N_{\mathbf{0}},N_{\ba^\star})).
\end{align*}
We use $\{N_{\ba^\star}(s_i,a,s')\}|_{p,\mathrm{ALG}}$ to denote the probability measure of random variable $\{N_{\ba^\star}(s_i,a,s')\}$ when the transition probability is $p$ and the executing algorithm is ALG. By a change of measure approach, we have:
\begin{align*}
    &\quad\mathbb{E}_{\ba^\star\sim \nu}\left[L_{h,i}(\hat{\pi}, \ba^\star)\right] \ge \inf_{f} \mathbb{E}_{\ba^\star\sim \nu}\mathbb{E}_{N_{\ba^\star}\sim (p_{\ba^\star}, \mathrm{ALG}), N_{\mathbf{0}}\sim (p_{\mathbf{0}}, \mathrm{ALG})}(1-f(N_{\mathbf{0}},N_{\ba^\star}))  \\
    &\ge \inf_{f} \mathbb{E}_{\ba^\star\sim \nu}\mathbb{E}_{N_{\ba^\star}\sim (p_{\mathbf{0}}, \mathrm{ALG}), N_{\mathbf{0}}\sim (p_{\mathbf{0}}, \mathrm{ALG})}(1-f(N_{\mathbf{0}},N_{\ba^\star})) \\
    &\qquad -\frac{1}{2}\mathbb{E}_{\ba^\star\sim \nu} \mathbb{E}_{N_{\mathbf{0}}\sim (p_{\mathbf{0}}, \mathrm{ALG})}\sup_{f}\sum_{\mathbf{n}\in\mathbb{N}^{A\times S}}|f(N_{\mathbf{0}},\mathbf{n})(\mathbb{P}_{N_{\ba^\star}\sim(p_{\ba^\star},\mathrm{ALG})}[N_{\ba^\star}=\mathbf{n}]-\mathbb{P}_{N_{\ba^\star}\sim(p_{\mathbf{0}},\mathrm{ALG})}[N_{\ba^\star}=\mathbf{n}])|\tag{change of measure}\\
    &\ge \inf_{f} \mathbb{E}_{\ba^\star\sim \nu}\mathbb{E}_{N_{\ba^\star}\sim (p_{\mathbf{0}}, \mathrm{ALG}), N_{\mathbf{0}}\sim (p_{\mathbf{0}}, \mathrm{ALG})}(1-f(N_{\mathbf{0}},N_{\ba^\star})) \\
    &\qquad -\mathbb{E}_{\ba^\star\sim \nu} \mathrm{TV}\left(\{N_{\ba^\star}(s_i,a,s')\}|_{p_\mathbf{0},\mathrm{ALG}},\{N_{\ba^\star}(s_i,a,s')\}|_{p_{\ba^\star},\mathrm{ALG}}\right) \tag{definition of total variation}\\
    &= \inf_{f} \mathbb{E}_{N_{\ba^\star}\sim (p_{\mathbf{0}}, \mathrm{ALG}), N_{\mathbf{0}}\sim (p_{\mathbf{0}}, \mathrm{ALG})}\mathbb{E}_{\ba^\star\sim \nu}(1-f(N_{\mathbf{0}},N_{\ba^\star})) \\
    &\qquad -\mathbb{E}_{\ba^\star\sim \nu} \mathrm{TV}\left(\{N_{\ba^\star}(s_i,a,s')\}|_{p_\mathbf{0},\mathrm{ALG}},\{N_{\ba^\star}(s_i,a,s')\}|_{p_{\ba^\star},\mathrm{ALG}}\right) \tag{$f$ is bounded and measurable, Fubini's theorem}\\
    &= \inf_{f} \mathbb{E}_{N_{\ba^\star}\sim (p_{\mathbf{0}}, \mathrm{ALG}), N_{\mathbf{0}}\sim (p_{\mathbf{0}}, \mathrm{ALG})}\frac{1}{A}\sum_{k=1}^A[(1-f(N_{\mathbf{0}},N_{\ba^\star}))\mid a_i^\star=k] \\
    &\qquad -\mathbb{E}_{\ba^\star\sim \nu} \mathrm{TV}\left(\{N_{\ba^\star}(s_i,a,s')\}|_{p_\mathbf{0},\mathrm{ALG}},\{N_{\ba^\star}(s_i,a,s')\}|_{p_{\ba^\star},\mathrm{ALG}}\right) \tag{law of total expectation}\\
    &= \frac{A-1}{A}-\mathbb{E}_{\ba^\star\sim \nu} \mathrm{TV}\left(\{N_{\ba^\star}(s_i,a,s')\}|_{p_\mathbf{0},\mathrm{ALG}},\{N_{\ba^\star}(s_i,a,s')\}|_{p_{\ba^\star},\mathrm{ALG}}\right),
    \end{align*}
    where in the last equality we utilize the fact that $\sum_{k=1}^A[f(N_{\mathbf{0}},N_{\ba^\star})\mid a_i^\star=k]=\sum_{k=1}^A\hat{\pi}_h(k,s_i)=1$. This demonstrates that the information gain from the source MDP is bounded on average.
\paragraph{Link the average suboptimality gap to the sample complexity.}Since we have already linked the average loss function to the total variation distance, the rest we need to do is to link the total variation distance to the number of samples. To achieve this, by \citet[Lemma 15.1]{lattimore2020bandit}, we have:
    \begin{align*}
    &\quad\mathbb{E}_{\ba^\star\sim \nu}\left[L_{h,i}(\hat{\pi}, \ba^\star)\right]\ge \frac{A-1}{A}-\mathbb{E}_{\ba^\star\sim \nu}\sqrt{\frac{1}{2}\mathrm{KL}\left(\{N_{\ba^\star}(s_i,a,s')\}|_{p_\mathbf{0},\mathrm{ALG}},\{N_{\ba^\star}(s_i,a,s')\}|_{p_{\ba^\star},\mathrm{ALG}}\right)}\tag{Pinsker's inequality}\\
    &\ge \frac{A-1}{A}-\mathbb{E}_{\ba^\star\sim \nu}\sqrt{\frac{1}{2}\sum_{a=1}^A\mathbb{E}_{p_{\mathbf{0},\mathrm{ALG}}}[N_{\ba^\star}(s_i,a)]\cdot\mathrm{KL}(p_{\mathbf{0}}(\cdot\mid s_i,a),p_{\ba^\star}(\cdot\mid s_i,a))} \tag{\citet[Lemma 15.1]{lattimore2020bandit}}\\
    &\ge \frac{A-1}{A}-\mathbb{E}_{\ba^\star\sim \nu}\sqrt{\frac{1}{2}\mathbb{E}_{p_{\mathbf{0},\mathrm{ALG}}}[N_{\ba^\star}(s_i,a_i^\star)]\cdot\mathrm{KL}(p_{\mathbf{0}}(\cdot\mid s_i,a_i^\star),p_{\ba^\star}(\cdot\mid s_i,a_i^\star))} \tag{$p_{\mathbf{0}}$ and $p_{\ba^\star}$ only differ in $a_i^\star$ at $s_i$}\\
    &\ge \frac{A-1}{A}-\mathbb{E}_{\ba^\star\sim \nu}\sqrt{\mathbb{E}_{p_{\mathbf{0},\mathrm{ALG}}}[N_{\ba^\star}(s_i,a_i^\star)]\cdot\frac{2\gamma^2}{H}} \tag{$\mathrm{KL}(p_{\mathbf{0}}(\cdot\mid s_i,a),p_{\ba^\star}(\cdot\mid s_i,a))=\frac{1}{2H}\log(\frac{1}{1-4\gamma^2})\le \frac{4\gamma^2}{H}$ for $\gamma\le 1/3$}\\
    &\ge \frac{1}{2} -\sqrt{\frac{1}{A}\sum_{a=1}^A\mathbb{E}_{p_{\mathbf{0},\mathrm{ALG}}}[N_{\ba^\star}(s_i,a)]\cdot\frac{2\gamma^2}{H}} \tag{definition of $\nu$ and Jensen inequality}\\
    &= \frac{1}{2}-\sqrt{\frac{2\gamma^2}{HA}\mathbb{E}_{p_{\mathbf{0},\mathrm{ALG}}}[N_{\ba^\star}(s_i)]}.
\end{align*}
Therefore, with \cref{eq:sub-to-loss} we can link the average suboptimality gap to the number of samples from the target MDP $\mathcal{M}_{\ba^\star}$. 
\begin{align*}
    \mathbb{E}_{\ba^\star\sim\nu}\mathbb{E}_{\ba^\star}\left[  V_{1, M_{\ba^\star}}^\star - V_{1, M_{\ba^\star}}^{\hat{\pi}} \right]&\ge\frac{\gamma}{3S}\cdot \sum_{h=1}^H \sum_{i=1}^S(1-h/H)\mathbb{E}_{\ba^\star\sim\nu}L_{h,i}(\hat{\pi},\ba).\\
    &\ge \frac{\gamma(H-1)}{12}-\frac{\gamma(H-1)}{6}\sum_{i=1}^S\sqrt{\frac{2\gamma^2}{HA}\mathbb{E}_{p_{\mathbf{0},\mathrm{ALG}}}[N_{\ba^\star}(s_i)]} \\
    &\ge \frac{\gamma(H-1)}{6}\left(\frac{1}{2}-\sqrt{\frac{2\gamma^2}{SHA}\sum_{i=1}^S\mathbb{E}_{p_{\mathbf{0},\mathrm{ALG}}}[N_{\ba^\star}(s_i)]}\right) \\
    &\ge\frac{\gamma(H-1)}{6}(\frac{1}{2}-\sqrt{\frac{2\gamma^2n}{HSA}}),\tag{$\sum_{i=1}^S\mathbb{E}_{p_{\mathbf{0},\mathrm{ALG}}}[N_{\ba^\star}(s_i)]\le n$} 
\end{align*}
where $n$ denotes the number of samples from the target MDP $\mathcal{M}_{\ba^\star}$. Recall that $\gamma\le 48\varepsilon/H$, we take $\gamma=48\varepsilon/H\le 1/3$, for any $\varepsilon\le 1/48$, as long as $\sqrt{\frac{2\gamma^2n}{HSA}}\le \frac{1}{4}$, i.e.,
\begin{equation*}
    n\le \frac{HSA}{32\gamma^2}=\frac{H^3SA}{32\cdot 48^2\varepsilon^2},
\end{equation*}
we have the average suboptimality gap greater than $\varepsilon$:
\begin{equation*}
    \mathbb{E}_{\ba^\star\sim\nu}\mathbb{E}_{\ba^\star}\left[  V_{1, M_{\ba^\star}}^\star - V_{1, M_{\ba^\star}}^{\hat{\pi}} \right]\ge 2\varepsilon(H-1)/H\ge \varepsilon. \tag{$H\ge 2$}
\end{equation*}
Therefore, for any algorithm ALG, there exists an $\ba^\star\in[A]^S$, as long as $n\le\frac{H^3SA}{32\cdot 48^2\varepsilon^2}$, we have:
\begin{equation*}
    \mathbb{E}_{\ba^\star}\left[  V_{1, M_{\ba^\star}}^\star - V_{1, M_{\ba^\star}}^{\hat{\pi}} \right]\ge\varepsilon.
\end{equation*}
This ends the proof.
\end{proof}
\section{Proof of the Upper Bound}

\subsection{Preliminaries and Notations}
\paragraph{Empirical MDP.} Let $\{(s^t_h,a^t_h,(s^t_h)')\}_{h\in[H]}$ denote the observed transitions in the target MDP in the $t$-th episode. For any state-action pair $(s,a)\in \sa$, we let $\nt\triangleq \sum_{\tau=1}^t\sum_{h=1}^H\indic{(s^\tau_h,a^\tau_h)=(s,a)}$ denote the visitation count for $(s,a)$ in the first $t$ episodes in the target MDP, and correspondingly $n^{t}(s,a,s')\triangleq \sum_{\tau=1}^t\sum_{h=1}^H\indic{(s^\tau_h,a^\tau_h,(s^\tau_h)')=(s,a,s')}$. With these definitions, we can define the empirical transitions as:
\begin{equation*}
    \hat{p}_{\mathrm{tar}}^t(s'\mid s,a)\triangleq\frac{n^t(s,a,s')}{\nt}\,\text{if }\nt>0,
\end{equation*}
and $\hat{p}_{\mathrm{tar}}^t(s'\mid s,a)\triangleq 1/S$ otherwise. Besides, let $\hat{\mathcal{M}}_{\mathrm{tar}}^t\triangleq(\mathcal{S},\mathcal{A},H,\hat{p}_{\mathrm{tar}}^t,r,\rho)$ denote the empirical MDP. We denote the Q-function of a policy $\pi$ evaluated on $\hat{\mathcal{M}}_{\mathrm{tar}}$ by $Q_h^{\hat{p}_{\mathrm{tar}}^t,\pi}$, and the value function by $V_h^{\hat{p}_{\mathrm{tar}}^t,\pi}$.
\subsection{Proof of \cref{lem:sample-recognization}}\label{appendix:recognization}
\subsubsection{Good Events}
For \cref{alg:rf}, we now start by considering good events inspired by \citet{menard_fast_2021}.
 We consider the following good event $F^{\mathrm{RF}}$:
\begin{align*}
F_1 &\triangleq \left\{ \forall t \in \mathbb{N}^+, \forall (s, a) \in \sa : \mathrm{KL}(\hat{p}_{\mathrm{tar}}^t(\cdot \mid s, a), p_{\mathrm{tar}}(\cdot \mid s, a)) \leq \frac{g_1(n^t(s, a), \delta)}{n^t(s, a)} \right\},\\
F_2 &\triangleq \left\{ \forall t \in \mathbb{N}^+, \forall (s, a) \in \sa : n^t(s, a) \geq \frac{1}{2} \overline{n}_{\mathrm{tar}}^t(s, a) - g_3(\delta) \right\},\\
F_3 &\triangleq\left\{\forall (s, a) \in \sa:\dtv{p_{\mathrm{src}}(\cdot\mid s,a)}{\hat{p}_{\mathrm{src}}(\cdot\mid s,a)}\le  \beta/4\right\}, \\
F^{\mathrm{RF}}&\triangleq \bigcap_{i=1}^3 F_i,
\end{align*}
where the bonus functions are defined as follows:
\begin{align}
    g_1(n,\delta)&\triangleq\log(\frac{6SAH}{\delta})+S\log(8e(n+1)),\label{eq:g_1}\\
    g_3(\delta)&\triangleq\log(\frac{6SA}{\delta}).\label{eq:g_3}
\end{align}
Let $p_{h,\mathrm{tar}}^t(s,a)$ denote the probability of visiting $(s,a)$ at step $h$ in the $t$-th episode under $p_{\mathrm{tar}}$. Furthermore, let $\overline{n}^t_{\mathrm{tar}}(s,a)\triangleq \sum_{i=1}^t\sum_{h=1}^Hp_{h,\mathrm{tar}}^t(s,a)$ denote the cummulative probability of visiting $(s,a)$ in the first $t$ episodes in \cref{alg:rf}.
\begin{lemma}\label{lem:good-event1}
    We have $\mathbb{P}(F^{\mathrm{RF}})\ge 1-\delta/2$.
\end{lemma}
\begin{proof}
    By \cref{lem:kl-divergence}, we have $\mathbb{P}(F_1)\ge 1-\delta/6$. By \cref{lem:bernoulli}, we have $\mathbb{P}(F_2)\ge 1-\delta/6$. By \cref{lem:concentration-discrete} and the assumption on $\mathcal{D}_{\mathrm{src}}$ in \cref{lem:sample-recognization}, we have $\mathbb{P}(F_3)\ge 1-\delta/6$. By union bound, we have $\mathbb{P}(F)=\mathbb{P}(\bigcap_{i=1}^3 F_i)\ge 1-\delta/2$.
\end{proof}
\subsubsection{Key lemmas}
We adapt the reward-free exploration framework to directly control the transition estimation errors, enabling us to identify the shifted region $\mathcal{B}$ within finite samples. Specifically, we consider a special class $\mathcal{R}$ of reward function $r$ given by:
\begin{align*}
    \mathcal{R}\triangleq\{r:[H]\times\sa\times\mathcal{S} \mapsto[0,1]\mid \exists \text{ some } h^\star\in[H], \text{ if } h\neq h^\star, r_h(\cdot,\cdot,\cdot)=0\}.
\end{align*}
Concretely, any reward function $r\in \mathcal{R}$ not only depends on the current state-action, but also on possible next states. Moreover, $r$ is non-zero only at one step $h^\star$ (possibly different for different $r\in\mathcal{R}$). For $r\in\mathcal{R}$, for any policy $\pi$ and transition $p$, the Q-functions and value functions are defined as follows:
\begin{align*}
    Q_h^{p,\pi}(s,a)&\triangleq \sum_{s'}p(s'\mid s,a)(r_h(s,a,s')+V_{h+1}^{p,\pi}(s')),\,\forall (s,a)\in\sa, h\in[H], \\
    V_h^{p,\pi}(s)&\triangleq \max_{a}Q_h^{p,\pi}(s,a),\,\forall s\in\mathcal{S},h\in[H], \\
    Q_{H+1}^{p,\pi}(s,a)&\triangleq 0,V_{H+1}^{p,\pi}(s)\triangleq 0,\,\forall(s,a)\in\sa.
\end{align*}
The specific definition of $\mathcal{R}$ here helps us control the high-dimensional transition estimation errors for any $(s,a)$ and remove an $H$ factor from the sample complexity result. To achieve this goal, we define the uncertainty function (with $W_{H+1}^t(s,a)=0$):
\begin{equation}
\label{eq:uncertainty}
    W_h^{t}(s,a)=\min\left(1,\frac{4Hg_1(\nt,\delta)}{\nt}+\hat{p}_{\mathrm{tar}}^t\max_{a'\in\mathcal{A}}W_{h+1}^t(s,a)\right).
\end{equation}
We define the error term as:
\begin{equation*}
    e_h^{\pi,t}(s,a,r)\triangleq |Q_h^{p_{\mathrm{tar}},\pi}(s,a,r)-\hat{Q}_h^{\hat{p}_{\mathrm{tar}}^t,\pi}(s,a,r)|.
\end{equation*}
We first show that $W_1^{t}(s,a)$ can be utilized to establish an upper bound on $e_1^{\pi,t}(s,a,r)$ in the following lemma.
\begin{lemma}\label{lem:error-W}
    For any policy $\pi$ and any $r\in\mathcal{R}$, in any episode $t$ in \cref{alg:rf}, w.p. at least $1-\delta/2$, we have:
    \begin{equation*}
        [\rho \pi_1  e_1^{\pi,t}](r)\le 3\sqrt{\rho \pi_1  W_1^{t}}+\rho \pi_1  W_1^{t},
    \end{equation*}
\end{lemma}
where $[\rho \pi_1  e_1^{\pi,t}](r)=\sum_{s,a}\rho(s)\pi_1(a\mid s)e_1^{\pi,t}](s,a,r)$. The proof of \cref{lem:error-W} can be found in \cref{appendix:errow-W}. Next, we show that by choosing $\pi_1^{t+1}(s)=\arg\max_{a\in\mathcal{A}}W_1^{t}(s,a)$, $W_1^t$ can be used to upper bound $\dtv{p_{\mathrm{tar}}}{\hat{p}^t_{\mathrm{tar}}}$ uniformly for any $(s,a)\in\sa$, which helps us derive the stopping criterion for \cref{alg:rf}.
\begin{lemma}\label{lem:TV-upper}
    Under \cref{definition:reachability}, for any $\delta\in(0,1)$, w.p. $1-\delta/2$, for any $t\in\mathbb{N}$, any $(s,a)\in\mathcal{B}$, there exists a policy $\pi$ and reward $r\in\mathcal{R}$, such that:
    \begin{align*}
        \quad\dtv{p_{\mathrm{tar}}(\cdot\mid s,a)}{\hat{p}^t_{\mathrm{tar}}(\cdot\mid s,a)}\le \frac{2}{\sigma} [\rho \pi_1  e_1^{\pi,t}](r)\le \frac{2}{\sigma} (3\sqrt{\rho\max_{a\in\mathcal{A}}W_1^t}+\rho\max_{a\in\mathcal{A}}W_1^t),
    \end{align*}
    where $\rho\max_{a\in\mathcal{A}}W_1^t=\sum_{s\in\mathcal{S}}\rho(s)\max_{a\in\mathcal{A}}W_1^t(s,a)$.
\end{lemma}
\begin{proof}
    In this lemma, since only $\mathcal{M}_{\mathrm{tar}}$ is considered, we drop the subscript ``$\mathrm{tar}$'' for clarity. First we build an upper bound on $[\rho \pi_1  e_1^{\pi,t}](r)$ for any $\pi$ and $r\in\mathcal{R}$. By \cref{lem:error-W}, if we take $\pi_1^{t+1}(s)=\arg\max_{a\in\mathcal{A}}W_1^{t}(s,a)$, we can obtain that:
    \begin{equation}
        [\rho \pi_1  e_1^{\pi,t}](r)\le 3\sqrt{\rho \pi_1^{t+1}  W_1^{t}}+\pi^{t+1}_1\rho W_1^{t}.
    \end{equation}
    Now we build the connection between $ [\rho \pi_1  e_1^{\pi,t}](r)$ and $\dtv{p}{\hat{p}^t}$. First by the definition of $Q_1^{p,\pi}$, for $r\in\mathcal{R}$ we have:
    \begin{align*}
        \rho \pi_1  Q_1^{p,\pi}&=\mathbb{E}_{\pi}[\sum_{h=1}^H r_h(s_h,a_h,s_{h+1})] \\
        &=\mathbb{E}_{\pi}[r_{h^\star}(s_{h^\star},a_{h^\star},s_{h^\star+1})] \tag{only $r_{h^\star}$ is non-zero} \\
        &=\sum_{s,a,s'}p_{h^\star}^\pi(s,a)p(s'\mid s,a)r_{h^\star}(s,a,s'),
    \end{align*}
    where the last equality follows by recalling $p_{h^\star}^\pi(s,a)$ is the probability of reaching $(s,a)$ at step $h^\star$ under policy $\pi$. Therefore, by triangle inequality, we have:
    \begin{align}
    \rho \pi_1  e_1^{\pi,t}(r)&\ge|\rho \pi_1  Q_1^{p,\pi}-\rho \pi_1  Q_1^{\hat{p}^t,\pi}| \nonumber\\ 
    &=|\sum_{s,a,s'}p_{h^\star}^\pi(s,a)p(s'\mid s,a)r_{h^\star}(s,a,s')-\sum_{s,a,s'}\hat{p}_{h^\star}^{\pi,t}(s,a)\hat{p}^t(s'\mid s,a)r_{h^\star}(s,a,s')|,\label{eq:e-r-pi}
    \end{align}
        In the following step, we fix a given $(s,a)$ for the analysis. 
    For a given $(s,a)$, we first take $(\pi',h^\star)=\arg\max_{\pi,h}p_{h}^\pi(s,a)$, and choose $r'\in\mathcal{R}$ as:
    \[
    r'_h(\overline{s},\overline{a},\cdot)=\left\{
    \begin{array}{ll}
        1 & \text{if } \overline{s}=s,\overline{a}=a,h=h^\star, \\
        0 & \text{otherwise} .
    \end{array}
    \right.
    \]
    Substituting $r'$ and $\pi'$ into \cref{eq:e-r-pi} yields that:
    \begin{equation}\label{eq:pi1}
        |p_{h^\star}^{\pi'}(s,a)-\hat{p}_{h^\star}^{\pi',t}(s,a)|\le \rho \pi'_1  e_1^{\pi',t}(r').
    \end{equation}
    Then we choose the reward function $r''\in\mathcal{R}$ to be:
    \[
    r''_h(\overline{s},\overline{a},\overline{s}')=\left\{
    \begin{array}{ll}
        1 & \text{if } \overline{s}=s,\overline{a}=a,h=h^\star,p(\overline{s}'\mid s,a)>\hat{p}^t(\overline{s}'\mid s,a), \\
        0 & \text{otherwise} .
    \end{array}
    \right.
    \]
    Substituting $r''$ and $\pi'$ into \cref{eq:e-r-pi} yields that:
    \begin{align*}
        \rho \pi'_1  e_1^{\pi',t}(r'') &\ge\left|p_{h^\star}^{\pi'}(s,a)\sum_{s'}p(s'\mid s,a)r''_{h^\star}(s,a,s')-\hat{p}_{h^\star}^{\pi',t}(s,a)\sum_{s'}\hat{p}^t(s'\mid s,a)r''_{h^\star}(s,a,s')\right|\\
        &= \left|p_{h^\star}^{\pi'}(s,a)(\sum_{s'}p(s'\mid s,a)r''_{h^\star}(s,a,s')-\hat{p}^{t}(s'\mid s,a)r''_{h^\star}(s,a,s')) \right.\\
        &\qquad-\left.(\hat{p}_{h^\star}^{\pi',t}(s,a)-p_{h^\star}^{\pi'}(s,a))\sum_{s'}\hat{p}^{t}(s'\mid s,a)r''_{h^\star}(s,a,s')\right|\\
        &\ge p_{h^\star}^{\pi'}(s,a)|\sum_{s'}p(s'\mid s,a)r''_{h^\star}(s,a,s')-\hat{p}^{t}(s'\mid s,a)r''_{h^\star}(s,a,s')|\\
        &\qquad-|p_{h^\star}^{\pi'}(s,a)-\hat{p}_{h^\star}^{\pi',t}(s,a)||\hat{p}^tr''_{h^\star}(s,a)| \tag{triangle inequality}\\
        &= p_{h^\star}^{\pi'}(s,a)\sum_{s':p(s'\mid s,a)>\hat{p}^{t}(s'\mid s,a)}(p(s'\mid s,a)-\hat{p}^{t}(s'\mid s,a))\\
        &\qquad-|p_{h^\star}^{\pi'}(s,a)-\hat{p}_{h^\star}^{\pi',t}(s,a)||\hat{p}^tr''_{h^\star}(s,a)| \tag{Definition of $r''$}\\
        &= p_{h^\star}^{\pi'}(s,a)\dtv{p(\cdot\mid s,a)}{\hat{p}^t(\cdot\mid s,a)}-|p_{h^\star}^{\pi'}(s,a)-\hat{p}_{h^\star}^{\pi',t}(s,a)||\hat{p}^tr''_{h^\star}(s,a)| \\
        &\ge \sigma\dtv{p(\cdot\mid s,a)}{\hat{p}^t(\cdot\mid s,a)}-\rho \pi'_1  e_1^{\pi',t}(r'), \tag{\cref{definition:reachability}. \cref{eq:pi1}}
    \end{align*}
    Thus, we obtain that:
    \begin{align*}
        \dtv{p(\cdot\mid s,a)}{\hat{p}^t(\cdot\mid s,a)}&\le \frac{1}{\sigma}(\rho \pi'_1  e_1^{\pi',t}(r')+\rho \pi'_1  e_1^{\pi',t}(r''))\\
        &\le \frac{2}{\sigma}\max_{r\in\{r',r''\}}\rho \pi'_1  e_1^{\pi',t}(r) \\
        &\le\frac{2}{\sigma} (3\sqrt{\rho \pi_1^{t+1}  W_1^{t}}+\rho \pi_1^{t+1}  W_1^{t}), \tag{\cref{lem:error-W}}
    \end{align*}
    where $\pi_1^{t+1}(s)=\arg\max_{a\in\mathcal{A}}W_1^{t}(s,a)$ as we said before. This ends the proof and yields the choice of the stopping criterion of \cref{alg:rf}.
\end{proof}
\subsubsection{Proof of \cref{lem:sample-recognization}}
\begin{proof}
As a overview of the proof, we first show that the stopping criterion can guarantee that we have $\hat{\mathcal{B}}=\mathcal{B}$. Then we solve for the upper bound of the stopping time $\tau$ to get the sample complexity of \cref{alg:rf}.
\paragraph{Identify the shift successfully.}By the stopping criterion in \cref{alg:rf} and \cref{lem:TV-upper}, we have for any $(s,a)\in\sa$:
    \begin{equation*}
        \dtv{p_{\mathrm{tar}}(\cdot\mid s,a)}{\hat{p}^t_{\mathrm{tar}}(\cdot\mid s,a)}\le\frac{2}{\sigma} (3\sqrt{\rho \pi_1^{t+1}  W_1^{t}}+\rho \pi_1^{t+1}  W_1^{t})\le \beta/4.
    \end{equation*}
    On event $F^{\mathrm{RF}}$, we have:
    \begin{align*}
        \dtv{p_{\mathrm{src}}(\cdot\mid s,a)}{\hat{p}_{\mathrm{src}}(\cdot\mid s,a)}\le \beta/4,
    \end{align*}
    which implies that if $\dtv{\hat{p}_{\mathrm{tar}}^{t}(\cdot\mid s,a)}{\hat{p}_{\mathrm{src}}(\cdot\mid s,a)}> \beta/2$, we have:
    \begin{align*}
        \dtv{p_{\mathrm{src}}(\cdot\mid s,a)}{p_{\mathrm{tar}}(\cdot\mid s,a)}&\ge\dtv{\hat{p}_{\mathrm{tar}}^{t}(\cdot\mid s,a)}{\hat{p}_{\mathrm{src}}(\cdot\mid s,a)}-\dtv{p_{\mathrm{tar}}(\cdot\mid s,a)}{\hat{p}_{\mathrm{tar}}^t(\cdot\mid s,a)} \\
        &\quad- \dtv{p_{\mathrm{src}}(\cdot\mid s,a)}{\hat{p}_{\mathrm{src}}(\cdot\mid s,a)}\tag{triangle inequality} \\
        &>\beta/2 -\beta/4-\beta/4=0.
    \end{align*}
    This means that $\hat{\mathcal{B}}\subseteq \mathcal{B}$. On the contrary, if $\dtv{p_{\mathrm{src}}(\cdot\mid s,a)}{p_{\mathrm{tar}}(\cdot\mid s,a)}\ge \beta$, we have:
    \begin{align*}
        \dtv{\hat{p}^{t}_{\mathrm{tar}}(\cdot\mid s,a)}{\hat{p}_{\mathrm{src}}(\cdot\mid s,a)}&\ge \dtv{p_{\mathrm{src}}(\cdot\mid s,a)}{p_{\mathrm{tar}}(\cdot\mid s,a)}-\dtv{p_{\mathrm{tar}}(\cdot\mid s,a)}{\hat{p}_{\mathrm{tar}}^t(\cdot\mid s,a)} \\
        &\quad- \dtv{p_{\mathrm{src}}(\cdot\mid s,a)}{\hat{p}_{\mathrm{src}}(\cdot\mid s,a)}\tag{triangle inequality} \\
        &\ge \beta-\beta/4-\beta/4=\beta/2,
    \end{align*}
    which implies that $ \mathcal{B}\subseteq\hat{\mathcal{B}}$. Thus, by \cref{lem:good-event1}, we have $ \mathcal{B}=\hat{\mathcal{B}}$ w.p. $1-\delta/2$. This ends the proof of the first part.
    
\paragraph{Solve for the sample complexity.}Next we solve for the sample complexity \cref{alg:rf}. In the following steps, since only $\mathcal{M}_{\mathrm{tar}}$ is considered, we drop the subscript `tar' for clarity.

As an overview of this part, we first upper bound $\rho \pi_1^{t+1}  W_1$ with the weighted sum of exploration bonus under the true visitation probability. Further, we control this term by change the visitation count to the cumulative visitation probability. Summing these upper bounds over $t< \tau$, we get an inequality for $\tau$. Solving this inequality, we finally get an upper bound of $\tau$, and in turn the sample complexity.

To begin with, we establish an upper bound on $W_h^t(s,a)$ for all $(s,a,h)$. If $\nt>0$, we have:
    \begin{align*}
        W_h^t(s,a)&\le 4\frac{Hg_1(\nt,\delta)}{\nt}+\hat{p}^{t}\max_{a'\in\mathcal{A}}W_{h+1}^t(s,a) \\
        &= 4\frac{Hg_1(\nt,\delta)}{\nt}+p\pi_{h+1}^{t+1}W_{h+1}^t(s,a)+(\hat{p}^t-p)\pi_{h+1}^{t+1}W_{h+1}^t(s,a) \\
        &\le 4\frac{Hg_1(\nt,\delta)}{\nt}+p\pi_{h+1}^{t+1}W_{h+1}^t(s,a)+\sqrt{2\revar(\pi_{h+1}^{t+1}W_{h+1}^t)(s,a)\frac{g_1(\nt,\delta)}{\nt}}+\frac{2}{3}\frac{g_1(\nt,\delta)}{\nt} \tag{\cref{lem:trans}} \\
        &\le 5\frac{Hg_1(\nt,\delta)}{\nt}+p\pi_{h+1}^{t+1}W_{h+1}^t(s,a)+\sqrt{2\frac{1}{H}p(\pi_{h+1}^{t+1}W_{h+1}^t)(s,a)H\frac{g_1(\nt,\delta)}{\nt}} \tag{$\revar(\pi_{h+1}^{t+1}W_{h+1}^t)(s,a)\le p\pi_{h+1}^{t+1}W_{h+1}^t(s,a)$}\\
        &\le 7H\frac{g_1(\nt,\delta)}{\nt}+(1+\frac{1}{H})p\pi_{h+1}^{t+1}W_{h+1}^t(s,a). \tag{$\sqrt{xy}\le x+y$}
    \end{align*}
    Since $W_h^t(s,a)\le 1$, we have for all $\nt\ge 0$:
    \begin{equation*}
         W_h^t(s,a)\le 7H(\frac{g_1(\nt,\delta)}{\nt}\land 1)+(1+\frac{1}{H})p\pi_{h+1}^{t+1}W_{h+1}^t(s,a). 
    \end{equation*}
    Apply this inequality recursively, we get:
    \begin{align*}
        \rho \pi_1^{t+1}  W_1^t\le 7eH\sum_{h=1}^H\sum_{s,a}p_{h}^{t+1}(s,a)(\frac{g_1(\nt,\delta)}{\nt}\land 1).\tag{$(1+\frac{1}{H})^h\le e$}
    \end{align*}
    Suppose the stopping time is $\tau$, by the stopping criterion and \cref{eq:error-upper-bound} we have for $t<\tau$,
    \begin{align*}
        \sigma\beta/8\le 9\sqrt{eH\sum_{h=1}^H\sum_{(s,a)\in\sa}p_h^{t+1}(\frac{g_1(\nt,\delta)}{\nt} \land 1)}+7eH\sum_{h=1}^H\sum_{(s,a)\in\sa}p_h^{t+1}\frac{g_1(\nt,\delta)}{\nt}\land 1.
    \end{align*}
    Summing over all $t<\tau$, we obtain that,
    \begin{align*}
        (\tau-1)\sigma\beta&\le 72\sum_{t=1}^{\tau-1}\sqrt{eH\sum_{h=1}^H\sum_{(s,a)\in\sa}p_h^{t+1}(\frac{g_1(\nt,\delta)}{\nt} \land 1)}+56eH\sum_{t=1}^{\tau-1}\sum_{h=1}^H\sum_{s,a}p_h^{t+1}\frac{g_1(\nt,\delta)}{\nt}\land 1 \\
        &\le 72\sqrt{eH(\tau-1)\sum_{t=1}^{\tau-1}\sum_{h=1}^H\sum_{(s,a)\in\sa}p_h^{t+1}(\frac{g_1(\nt,\delta)}{\nt} \land 1)}\\
        &\qquad+56eH\sum_{t=1}^{\tau-1}\sum_{h=1}^H\sum_{(s,a)\in\sa}p_h^{t+1}\frac{g_1(\nt,\delta)}{\nt}\land 1 \tag{Cauchy-Schwarz inequality}\\
        &\le 72\sqrt{eH\tau\sum_{t=1}^{\tau-1}\sum_{h=1}^H\sum_{(s,a)\in\sa}p_h^{t+1}(\frac{g_1(\overnt,\delta)}{\overnt\lor 1} )}+56eH\sum_{t=1}^{\tau-1}\sum_{h=1}^H\sum_{(s,a)\in\sa}p_h^{t+1}\frac{g_1(\overnt,\delta)}{\overnt\lor 1} \tag{\cref{lem:occupancy}} \\
        &\le 72\sqrt{eH\tau g_1(\tau-1,\delta)\sum_{t=1}^{\tau-1}\sum_{s,a}\frac{\overline{n}^{t+1}(s,a)-\overnt}{\overnt\lor 1}}+56eHg_1(\tau-1,\delta)\sum_{t=1}^{\tau-1}\sum_{s,a}\frac{\overline{n}^{t+1}(s,a)-\overnt}{\overnt\lor 1}\\
        &\le 144\sqrt{eH\tau g_1(\tau-1,\delta)SA\log(\tau+1)}+224eHSAg_1(\tau-1,\delta)\log(\tau+1) \tag{\cref{lem:summation}} \\
        &\le 144\sqrt{eH\tau SA(\log(\frac{6SAH}{\delta}\log(8e
        \tau)+S\log^2(8e\tau)))}+224eHSA(\log(\frac{6SAH}{\delta}\log(8e
        \tau)+S\log^2(8e\tau))) \tag{$\log(\tau+1)\le\log(8e\tau)$}.
    \end{align*} 
    Solving this inequality by \cref{lem:solveineq}, we obtain that:
    \begin{equation*}
        \tau\le C_1\frac{HSA}{(\sigma\beta)^2},
    \end{equation*}
    where $C_1>0$ only contains log factors. Therefore, \cref{alg:rf} can identify $\mathcal{B}$ using at most $C_1\frac{H^2SA}{(\sigma\beta)^2}$ from the target MDP. This concludes the proof.
    \end{proof}
    \begin{remark}
        Note that directly applying the proof technique in \citet{menard_fast_2021} or any other sample-efficient algorithms cannot establish bounds on the estimation error of transitions and in turn identify the shifted region. This can be illustrated through the example in \cref{fig:RF-example}. Besides, in terms of sample complexity, simply apply the proof technique in \citet{menard_fast_2021} will result in an additional $H$ factor in the sample complexity result, while we improve upon this in our problem by considering a more general reward functions.
    \end{remark}
\begin{figure}
    \centering
    \includegraphics[width=0.5\linewidth]{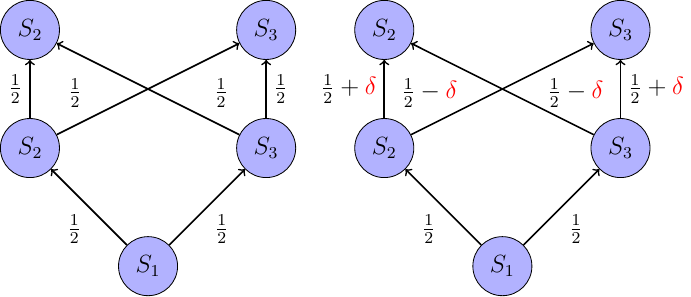}
    \caption{The MDP on the left is the real MDP, and the MDP on the right is the empirical MDP. For any reward function $r(\cdot,\cdot)$ and any policy $\pi$, the simulation error $|Q_1-\hat{Q}_1|$ is zero, while the estimation error can be significant with $\delta$.}
    \label{fig:RF-example}
\end{figure}
\subsection{Proof of \cref{theorem:sample-complexity}}\label{appendix:main-result}
\subsubsection{Good Events}
Since the sample complexity of the first step in \cref{alg:hybrid} is already studied in \cref{lem:sample-recognization}, here we mainly focus on the episodes in the second step (i.e., \cref{alg:ucbvi}). We use the index $t$ to denote the number of episodes in the second step of \cref{alg:hybrid}. For \cref{alg:hybrid}, we consider the following good event $F^{\mathrm{Hybrid}}$:
\begin{align*}
F_4 &\triangleq \left\{ \forall t \in \mathbb{N}^+, \forall (s, a) \in \sa : \mathrm{KL}(\tpt(\cdot \mid s, a), p_{\mathrm{tar}}(\cdot \mid s, a)) \leq \frac{g_1(\tnt(s,a), \delta)}{\tnt(s, a)} \right\},\\
F_5 &\triangleq \left\{ \forall t \in \mathbb{N}^+, \forall (s, a) \in \hat{\mathcal{B}} : \tnt(s, a) \geq \frac{1}{2} \overline{n}_{\mathrm{tar}}^t(s, a) - g_3(\delta) \right\},\\
F_6 &\triangleq \bigg\{
\forall t \in \mathbb{N}^+, \forall h \in [H], \forall (s, a) \in \sa, \\
& \ \ \ \ \left| (\tpt - p_{\mathrm{tar}}) V_{h+1}^{p_{\mathrm{tar}},\star}(s, a) \right| \leq \min \bigg( H, \sqrt{2 \mathrm{Var}_{p_{\mathrm{tar}}}(V_{h+1}^{p_{\mathrm{tar}},\star})(s, a) \frac{g_2(\tnt(s, a), \delta)}{\tnt(s, a)}} + \frac{3H g_2(\tnt(s, a), \delta)}{\tnt(s, a)} \bigg) 
\bigg\},\\
F^{\mathrm{Hybrid}} &\triangleq F^{\mathrm{RF}}\bigcap\bigg( \bigcap_{i=4}^6 F_i\bigg),\end{align*}
where the bonus functions $g_1$ and $g_3$, event $F^{\mathrm{RF}}$, are defined in \cref{appendix:recognization}, and $g_2$ is defined as follows:
\begin{align*}
    g_2(n,\delta)&\triangleq\log(\frac{6SAH}{\delta})+\log(8e(n+1)).
\end{align*}
In the definition of $F^{\mathrm{Hybrid}}$,
\[
\tnt(s,a) \triangleq
\begin{cases}
    n^t(s,a), & \text{for } (s,a)\in\hat{\mathcal{B}}, \\
    n_{\mathrm{src}}(s,a), & \text{for } (s,a)\notin\hat{\mathcal{B}},
\end{cases}
\]

\[
\tpt(\cdot\mid s,a) \triangleq
\begin{cases}
    \hat{p}_{\mathrm{tar}}^t(\cdot\mid s,a), & \text{for } (s,a)\in\hat{\mathcal{B}}, \\
    \hat{p}_{\mathrm{src}}(\cdot\mid s,a), & \text{for } (s,a)\notin\hat{\mathcal{B}},
\end{cases}
\]
where $\nt$ denotes the visitation count in \cref{alg:ucbvi}.
Let $p_{h,\mathrm{tar}}^t(s,a)$ denote the probability of visiting $(s,a)$ at step $h$ in the $t$-th episode under the true dynamics. Furthermore, let $\overline{n}^t_{\mathrm{tar}}(s,a)\triangleq \sum_{i=1}^t\sum_{h=1}^Hp_{h,\mathrm{tar}}^t(s,a)$ denote the cumulative probability of visiting $(s,a)$ in the first $t$ episodes. 
\begin{lemma}
    We have $\mathbb{P}(F^{\mathrm{Hybrid}})\ge 1-\delta$.
\end{lemma}
\begin{proof}
    By \cref{lem:sample-recognization}, on event $F^{\mathrm{RF}}$, we have $\hat{\mathcal{B}}=\mathcal{B}$, which implies that for $(s,a)\notin\hat{\mathcal{B}}$, $\tpt(\cdot\mid s,a)=\hat{p}_{\mathrm{src}}(\cdot\mid s,a)$ is the unbiased estimate of $p_{\mathrm{tar}}(\cdot\mid s,a)$, for $(s,a)\in\hat{\mathcal{B}}$, $\tpt(\cdot\mid s,a)=\hat{p}^t_{\mathrm{tar}}(\cdot\mid s,a)$, thus by \cref{lem:kl-divergence} we have $\mathbb{P}(F_4\mid F^{\mathrm{RF}})\ge 1-\delta/6$.\\
    By \cref{lem:bernoulli}, we have w.p. $1-\delta/6$,
    \begin{equation*}
        \tnt(s, a)=n^t(s,a) \geq \frac{1}{2} \overline{n}_{\mathrm{tar}}^t(s, a) - \log(\frac{6|\hat{\mathcal{B}}|}{\delta})\ge \frac{1}{2} \overline{n}_{\mathrm{tar}}^t(s, a) - \log(\frac{6SA}{\delta}),\,\forall t \in \mathbb{N}, \forall (s, a) \in \hat{\mathcal{B}},
    \end{equation*}
    thus we have $\mathbb{P}(F_5\mid F^{\mathrm{RF}})\ge 1-\delta/6$. \\
    Similarly, on event $F^{\mathrm{RF}}$, we have $\hat{\mathcal{B}}=\mathcal{B}$, which implies that for $(s,a)\notin\hat{\mathcal{B}}$, $\tpt(\cdot\mid s,a)=\hat{p}_{\mathrm{src}}(\cdot\mid s,a)$ is the unbiased estimate of $p_{\mathrm{tar}}(\cdot\mid s,a)$, for $(s,a)\in\hat{\mathcal{B}}$, $\tpt(\cdot\mid s,a)=\hat{p}^t_{\mathrm{tar}}(\cdot\mid s,a)$, thus by \cref{lem:bernstein} we have $\mathbb{P}(F_6\mid F^{\mathrm{RF}})\ge 1-\delta/6$. By union bound, we have $\mathbb{P}(F^{\mathrm{Hybrid}})=\mathbb{P}(\bigcap_{i=4}^6 F_i\mid F^{\mathrm{RF}})\cdot\P(F^{\mathrm{RF}})\ge 1-\delta$.
\end{proof}
\begin{remark}
    $F_4$ differs from $F_1$ in that $F_4$ considers the sequence of $\tpt$ in the execution of \cref{alg:ucbvi}, while $F_1$ considers the sequence of $\hat{p}^t$ in the execution of \cref{alg:rf}. A similar observation holds for $F_2$ and $F_5$.
\end{remark}
\subsubsection{Key Lemmas}
We define the following confidence intervals for the optimal value functions,
\begin{align}
\overline{Q}_h^t(s, a) &\triangleq \min \left( 
    H, r(s, a) + 3 \sqrt{\mathrm{Var}_{\tpt} (\overline{V}_{h+1}^t)(s, a) \frac{g_2(\tnt(s, a), \delta)}{\tnt(s, a)}} + 14H^2 \frac{g_1(\tnt(s, a), \delta)}{\tnt(s, a)} \right.\nonumber\\
    &\qquad \left. + \frac{1}{H} \tpt (\overline{V}_{h+1}^t - \underline{V}_{h+1}^t)(s, a) + \tpt \overline{V}_{h+1}^t(s, a) \right), \forall\, (s,a)\in\sa,\label{eq:overlineQ}\\
\overline{V}_h^t(s) &\triangleq \max_{a \in \mathcal{A}} \overline{Q}_h^t(s, a), \\
\overline{V}_{H+1}^t(s) &\triangleq 0, \\
\underline{Q}_h^t(s, a) &\triangleq \max \left( 
    0, r(s, a) - 3 \sqrt{\mathrm{Var}_{\tpt} (\overline{V}_{h+1}^t)(s, a) \frac{g_2(\tnt(s, a), \delta)}{\tnt(s, a)}} - 14H^2 \frac{g_1(\tnt(s, a), \delta)}{\tnt(s, a)} \right. \nonumber\\
    &\qquad \left. - \frac{1}{H} \tpt (\overline{V}_{h+1}^t - \underline{V}_{h+1}^t)(s, a) + \tpt \underline{V}_{h+1}^t(s, a) \right), \forall\, (s,a)\in\sa,\\
\underline{V}_h^t(s) &\triangleq \max_{a \in \mathcal{A}} \underline{Q}_h^t(s, a), \\
\underline{V}_{H+1}^t(s) &\triangleq 0.
\end{align}
\begin{lemma}[Algorithm ensures optimism]
\label{lem:optimism}
On event $F^{\mathrm{Hybrid}}$, we have that for all $t$, all $ h \in [H] $, and any $(s, a)\in\sa$,
\begin{equation}
\label{eq:optimism-q}
\underline{Q}_h^t(s, a) \leq Q_h^{p_{\mathrm{tar}},\star}(s, a) \leq \overline{Q}_h^t(s, a),
\end{equation}
and
\begin{equation}
\label{eq:optimism-v}
\underline{V}_h^t(s) \leq V_h^{p_{\mathrm{tar}},\star}(s) \leq \overline{V}_h^t(s).    
\end{equation}

\end{lemma}
The proof of \cref{lem:optimism} can be found in \cref{appendix:optimism}. We further define the following value functions to bound $V_h^{p_{\mathrm{tar}},\pi^{t+1}}$:
\begin{align*}
\tilde{Q}_h^t(s, a) &\triangleq \min \left(r(s,a)+p_{\mathrm{tar}}\tilde{V}_h^t(s,a),\max \left( 
    0, r(s, a) - 3 \sqrt{\mathrm{Var}_{\tpt} (\overline{V}_{h+1}^t)(s, a) \frac{g_2(\tnt(s, a), \delta)}{\tnt(s, a)}} \right. \right.  \\
    &\qquad\left. \left. - 14H^2 \frac{g_1(\tnt(s, a), \delta)}{\tnt(s, a)}  - \frac{1}{H} \tpt (\overline{V}_{h+1}^t - \underline{V}_{h+1}^t)(s, a) + \tpt \tilde{V}_{h+1}^t(s, a) ) \right)\right), \forall\, (s,a)\in\sa,\\
\tilde{V}_h^t(s) &\triangleq \pi_{h}^{t+1} \tilde{Q}_h^t(s, a), \\
\tilde{V}_{H+1}^t(s) &\triangleq 0.
\end{align*}
With this definition, we are ready to lower bound $V_h^{p_{\mathrm{tar}},\pi^{t+1}}$ in the following lemma.
\begin{lemma}
\label{lem:lower-estimate}
On event $F^{\mathrm{Hybrid}}$, we have that for all \((s, a, h)\),
\begin{align*}
\tilde{Q}_h^t(s, a) &\leq \min\left( \underline{Q}_h^t(s, a), Q_h^{p_{\mathrm{tar}},\pi^{t+1}}(s, a) \right), \\
\tilde{V}_h^t(s) &\leq \min\left( \underline{V}_h^t(s), V_h^{p_{\mathrm{tar}},\pi^{t+1}}(s) \right).
\end{align*}
\end{lemma}
\begin{proof}
    We proceed by induction. For the base case, we have $\tilde{Q}_{H+1}^t(s,a)=\underline{Q}_{H+1}^t(s,a)=Q_{H+1}^{p_{\mathrm{tar}},\pi^{t+1}}(s,a)=0$, so the claim trivially holds. Assuming the statement holds for step $h+1$, then for step $h$, we first prove that $\tilde{Q}_{h}^t(s,a)\le Q_h^{p_{\mathrm{tar}},\pi^{t+1}}(s, a)$. With the induction hypothesis, we have:
    \begin{equation*}
        Q_h^{p_{\mathrm{tar}},\pi^{t+1}}(s, a)-\tilde{Q}_{h}^t(s,a)\ge p_{\mathrm{tar}}(V_{h+1}^{p_{\mathrm{tar}},\pi^{t+1}}-\tilde{V}_{h+1}^t)(s,a)\ge 0.
    \end{equation*}
    Then by definition, we have:
    \begin{equation*}
        \tilde{V}_h^t(s)=\max_{a\in\mathcal{A}}\tilde{Q}_{h}^t(s,a)\le \max_{a\in\mathcal{A}}Q_h^{p_{\mathrm{tar}},\pi^{t+1}}(s, a)=V_{h}^{p_{\mathrm{tar}},\pi^{t+1}}(s).
    \end{equation*}
    This ends the proof of the first part. For the second part, we have: 
    \begin{align}
    \tilde{Q}_h^t(s,a)&\le\max \left( 
    0, r(s, a) - 3 \sqrt{\mathrm{Var}_{\tpt} (\overline{V}_{h+1}^t)(s, a) \frac{g_2(\tnt(s, a), \delta)}{\tnt(s, a)}} \right.  \nonumber \\
    &\qquad\left.  - 14H^2 \frac{g_1(\tnt(s, a), \delta)}{\tnt(s, a)}  - \frac{1}{H} \tpt (\overline{V}_{h+1}^t - \underline{V}_{h+1}^t)(s, a) + \tpt \tilde{V}_{h+1}^t(s, a) ) \right) \nonumber \\
    &\le\max \left( 
    0, r(s, a) - 3 \sqrt{\mathrm{Var}_{\tpt} (\overline{V}_{h+1}^t)(s, a) \frac{g_2(\tnt(s, a), \delta)}{\tnt(s, a)}} \right.  \nonumber\\
    &\qquad\left.  - 14H^2 \frac{g_1(\tnt(s, a), \delta)}{\tnt(s, a)}  - \frac{1}{H} \tpt (\overline{V}_{h+1}^t - \underline{V}_{h+1}^t)(s, a) + \tpt \underline{V}_{h+1}^t(s, a) ) \right) \nonumber \\
    &=\underline{Q}_h^{t}(s,a).
    \end{align}
    Then by definition we have:
    \begin{equation*}
        \tilde{V}_h^t(s)=\max_{a\in\mathcal{A}}\tilde{Q}_{h}^t(s,a)\le \max_{a\in\mathcal{A}}\underline{Q}_h^t(s, a)=\underline{V}_h^t(s).
    \end{equation*}
This ends the proof.
\end{proof}
\subsubsection{Proof of \cref{theorem:sample-complexity}}
\paragraph{Case 1: $\sigma\beta\ge\sqrt{\frac{S}{H}}\varepsilon$.}
Since by \cref{lem:sample-recognization}, we know that we can identify $\mathcal{B}$ using $\widetilde{O}(H^2SA/(\sigma\beta)^2)$ samples from the target MDP, so we focus on the sample complexity of the second step of \cref{alg:hybrid} (i.e., \cref{alg:ucbvi}). As an overview, we first show that the stopping criterion in \cref{alg:ucbvi} can yield an $\varepsilon$-optimal policy for the target MDP, then we study its sample complexity. 

\paragraph{Optimality guarantee.}We first define function $G_h^t$ to upper bound the performance gap $V_1^{p_{\mathrm{tar}},*}(\rho)-V_1^{p_{\mathrm{tar}},\pi^{t+1}}(\rho)$. Specifically, let $G_{H+1}^t\triangleq 0$ for all $(s,a)\in \mathcal{S}\times\mathcal{A}$, and for all $h\in [H]$,
\begin{align}
G_h^t(s, a) \triangleq& \min \left( 
    H, 6 \sqrt{ \mathrm{Var}_{\tpt} (\overline{V}_{h+1}^t)(s, a) \frac{g_2(\tnt(s, a), \delta)}{\tnt(s, a)}} 
    + 35H^2 \frac{g_1(\tnt(s, a), \delta)}{\tnt(s, a)}\right. \nonumber\\
    &\left. + ( 1 + \frac{3}{H}) \tpt \pi_{h+1}^{t+1} G_{h+1}^t (s,a)\right),\,\forall (s,a)\in \sa, \label{eq:value-gap}
\end{align} 
We show in \cref{lem:value-gap-bound} that $G_1^t$ can be used to upper bound the optimality gap. The proof of \cref{lem:value-gap-bound} can be found in \cref{appendix:value-gap-bound}.
\begin{lemma}\label{lem:value-gap-bound}
    On event $F^{\mathrm{UCBVI}}$, for all t,
\begin{equation}
    V_1^{p_{\mathrm{tar}},\star}(\rho)-V_1^{p_{\mathrm{tar}},\pi^{t+1}}(\rho)\le\rho \pi_1^{t+1}  G_1^t.
\end{equation}
\end{lemma}
By \cref{lem:value-gap-bound}, the stopping criterion in \cref{alg:ucbvi} will yield an $\varepsilon$-optimal policy for the target MDP. This ends the proof of this part.
\paragraph{Sample complexity of \cref{alg:ucbvi}.}As an overview of the proof for sample complexity result, we first link $G_1^t(s,a)$ with the summation of functions of the visitation count $\tnt(s,a)$ for all $(s,a)\in\sa$, then we divide the summation into two parts. For $(s,a)\notin \hat{\mathcal{B}}$, we can control the summation by the assumption of offline data. For $(s,a)\in \hat{\mathcal{B}}$, we change the visitation count $n^t(s,a)$ into the cumulative probability of visiting $(s,a)$ in the first t episodes and further control this term. Finally, to get the sample complexity result, we sum the value gap upper bound $\rho \pi_1^{t+1}  G_1^t$ over all $t<\tau$ and get an inequality of the stopping time $\tau$. By solving the inequality, we get an upper bound of $\tau$, and in turn the sample complexity result.
\paragraph{Link $G_1^t(s,a)$ to $\tnt(s,a)$.}
To link $G_1^t(s,a)$ to $\tnt(s,a)$, we need to change the empirical transition into the true transition to make sure we can apply \cref{eq:value-gap} recursively. Then we replace the optimistic estimate $\overline{V}_{h+1}^t(s,a)$ with the true value function $V_{h+1}^{p_{\mathrm{tar}},\pi^t}$ to remove an $H$ factor. For the first part, we have:
\begin{align}
    |(\tpt-p)\pi_{h+1}^{t+1}G_{h+1}^t(s,a)|&\le \sqrt{2\tarvar(\pi_{h+1}^{t+1}G_{h+1}^t)(s,a)\frac{g_1(\tnt(s,a),\delta)}{\tnt(s,a)}}+\frac{2}{3}H\frac{g_1(\tnt(s,a),\delta)}{\tnt(s,a)} \nonumber \tag{\cref{lem:trans}} \\
    &\le \frac{1}{8H^2}\tarvar(\pi_{h+1}^{t+1}G_{h+1}^t)(s,a)+17H^2\frac{g_1(\tnt(s,a),\delta)}{\tnt(s,a)} \nonumber \tag{$\sqrt{xy}\le x+y$} \\
    &\le \frac{1}{8H}p_{\mathrm{tar}}\pi_{h+1}^{t+1}G_{h+1}^t(s,a)+17H^2\frac{g_1(\tnt(s,a),\delta)}{\tnt(s,a)} \nonumber \tag{$0\le \pi_{h+1}^{t+1}G_{h+1}^t(s,a)\le H$}.
\end{align}
For the second part we have:
\begin{align}
    \tmvar(\overline{V}_{h+1}^t)(s,a)&\le 2\tarvar(\overline{V}_{h+1}^t)(s,a)+4H^2\frac{g_1(\tnt(s,a),\delta)}{\tnt(s,a)} \nonumber \tag{\cref{lem:var-pq}} \\
    &\le 4\tarvar(V_{h+1}^{p_{\mathrm{tar}},\pi^{t+1}})(s,a)+4Hp_{\mathrm{tar}}(\overline{V}_{h+1}^t-V_{h+1}^{p_{\mathrm{tar}},\pi^{t+1}})(s,a)+4H^2\frac{g_1(\tnt(s,a),\delta)}{\tnt(s,a)} \nonumber \tag{\cref{lem:var-fg}} \\
    &\le 4\tarvar(V_{h+1}^{p_{\mathrm{tar}},\pi^{t+1}})(s,a)+4Hp_{\mathrm{tar}}\pi_{h+1}^{t+1}G_{h+1}^t(s,a)+4H^2\frac{g_1(\tnt(s,a),\delta)}{\tnt(s,a)} \nonumber \tag{\cref{lem:value-gap-bound}} 
\end{align}
Combining this, and by $\sqrt{x+y}\le \sqrt{x}+\sqrt{y}$ and $\sqrt{xy}\le x+y$, we have:
\begin{align}
    G_h^t(s,a)&\le 12 \sqrt{\tarvar(V_{h+1}^{p_{\mathrm{tar}},\pi^{t+1}})(s,a)\frac{g_2(\tnt(s,a),\delta)}{\tnt(s,a)}}+403H^2\frac{g_1(\tnt(s,a),\delta)}{\tnt(s,a)} \nonumber \\
    & \quad +(1+\frac{31}{8H}+\frac{1}{8H^2})p_{\mathrm{tar}}\pi_{h+1}^{t+1}G_{h+1}^t(s,a) \nonumber \\
    &\le 12 \sqrt{\tarvar(V_{h+1}^{p_{\mathrm{tar}},\pi^{t+1}})(s,a)\frac{g_2(\tnt(s,a),\delta)}{\tnt(s,a)}}+403H^2\frac{g_1(\tnt(s,a),\delta)}{\tnt(s,a)}+(1+\frac{4}{H})p_{\mathrm{tar}}\pi_{h+1}^{t+1}G_{h+1}^t(s,a) \nonumber
\end{align}
Since $0\le G_h^t(s,a)\le H$, then for all $\tnt(s,a)\ge 0$, the following inequality holds:
\begin{align}
    G_h^t(s,a)&\le 12 \sqrt{\tarvar(V_{h+1}^{p_{\mathrm{tar}},\pi^{t+1}})(s,a)(\frac{g_2(\tnt(s,a),\delta)}{\tnt(s,a)}\land 1)}+403H^2\frac{g_1(\tnt(s,a),\delta)}{\tnt(s,a)}\land 1 \nonumber \\
    &\qquad+(1+\frac{4}{H})p_{\mathrm{tar}}\pi_{h+1}^{t+1}G_{h+1}^t(s,a)
\end{align}
Applying this inequality recursively, and notice that $(1+\frac{1}{x})^x\le e$, we have:
\begin{align}
    \rho \pi_1^{t+1}  G_1^t&\le 12e^4 \sum_{h=1}^H\sum_{(s,a)\in\sa}p_{h,\mathrm{tar}}^{t+1}\sqrt{\tarvar(V_{h+1}^{\pi^{t+1}})(s,a)(\frac{g_2(\tnt(s,a),\delta)}{\tnt(s,a)}\land 1)} \nonumber \\
    &\qquad+403e^{4}H^2\sum_{h=1}^H\sum_{(s,a)\in\sa}p_{h,\mathrm{tar}}^{t+1}\frac{g_1(\tnt(s,a),\delta)}{\tnt(s,a)}\land 1 
\end{align}
Furthermore, by \cref{lem:tv} and Cauchy-Schwarz inequality, for $(s,a)\in\hat{\mathcal{B}}$, we obtain:
\begin{align}
    \sum_{h=1}^H\sum_{(s,a)\in \hat{\mathcal{B}}}&p_{h,\mathrm{tar}}^{t+1}\sqrt{\tarvar(V_{h+1}^{p_{\mathrm{tar}},\pi^{t+1}})(s,a)(\frac{g_2(\tnt(s,a),\delta)}{\tnt(s,a)}\land 1)} \nonumber \\
    &\le \sqrt{\sum_{h=1}^H\sum_{(s,a)\in \hat{\mathcal{B}}}p_{h,\mathrm{tar}}^{t+1}\tarvar(V_{h+1}^{p_{\mathrm{tar}},\pi^{t+1}})(s,a)}\sqrt{\sum_{h=1}^H\sum_{(s,a)\in \hat{\mathcal{B}}}p_{h,\mathrm{tar}}^{t+1}(\frac{g_2(\nt,\delta)}{\nt}\land 1)} \nonumber \tag{Cauchy-Swarchz inequality}\\
    &\le\sqrt{\sum_{h=1}^H\sum_{(s,a)\in\sa}p_{h,\mathrm{tar}}^{t+1}\tarvar(V_{h+1}^{p_{\mathrm{tar}},\pi^{t+1}})(s,a)}\sqrt{\sum_{h=1}^H\sum_{(s,a)\in \hat{\mathcal{B}}}p_{h,\mathrm{tar}}^{t+1}(\frac{g_2(\nt,\delta)}{\nt}\land 1)} \nonumber \\
    &=\sqrt{\mathbb{E}_{\pi^{t+1}}(\sum_{h=1}^H r_h-V_1^{p_{\mathrm{tar}},\pi^{t+1}})^2}\sqrt{\sum_{h=1}^H\sum_{(s,a)\in \hat{\mathcal{B}}}p_{h,\mathrm{tar}}^{t+1}(\frac{g_2(\nt,\delta)}{\nt}\land 1)} \nonumber \tag{\cref{lem:tv}}\\
    &\le H\sqrt{\sum_{h=1}^H\sum_{(s,a)\in \hat{\mathcal{B}}}p_{h,\mathrm{tar}}^{t+1}(\frac{g_2(\nt,\delta)}{\nt}\land 1)} 
\end{align}
Similarly, for $(s,a)\notin\hat{\mathcal{B}}$, we have:
\begin{align}
    \sum_{h=1}^H\sum_{(s,a)\notin \hat{\mathcal{B}}}&p_{h,\mathrm{tar}}^{t+1}\sqrt{\tarvar(V_{h+1}^{p_{\mathrm{tar}},\pi^{t+1}})(s,a)(\frac{g_2(\tnt(s,a),\delta)}{\tnt(s,a)}\land 1)} \nonumber \\
    &\le \sqrt{\sum_{h=1}^H\sum_{(s,a)\notin \hat{\mathcal{B}}}p_{h,\mathrm{tar}}^{t+1}\tarvar(V_{h+1}^{p_{\mathrm{tar}},\pi^{t+1}})(s,a)}\sqrt{\sum_{h=1}^H\sum_{(s,a)\notin \hat{\mathcal{B}}}p_{h,\mathrm{tar}}^{t+1}(\frac{g_2(\srcn,\delta)}{\srcn}\land 1)} \nonumber \tag{Cauchy-Swarchz inequality}\\
    &\le\sqrt{\sum_{h=1}^H\sum_{(s,a)\in\sa}p_{h,\mathrm{tar}}^{t+1}\tarvar(V_{h+1}^{p_{\mathrm{tar}},\pi^{t+1}})(s,a)}\sqrt{\sum_{h=1}^H\sum_{(s,a)\notin \hat{\mathcal{B}}}p_{h,\mathrm{tar}}^{t+1}(\frac{g_2(\srcn,\delta)}{\srcn}\land 1)} \nonumber \\
    &=\sqrt{\mathbb{E}_{\pi^{t+1}}(\sum_{h=1}^H r_h-V_1^{p_{\mathrm{tar}},\pi^{t+1}})^2}\sqrt{\sum_{h=1}^H\sum_{(s,a)\notin \hat{\mathcal{B}}}p_{h,\mathrm{tar}}^{t+1}(\frac{g_2(\srcn,\delta)}{\srcn}\land 1)} \nonumber \tag{\cref{lem:tv}}\\
    &\le H\sqrt{\sum_{h=1}^H\sum_{(s,a)\notin \hat{\mathcal{B}}}p_{h,\mathrm{tar}}^{t+1}(\frac{g_2(\srcn,\delta)}{\srcn}\land 1)} 
\end{align}
\paragraph{Control the summation separately.}
For $(s,a)\notin \hat{\mathcal{B}}$, by the assumption of \cref{theorem:sample-complexity} we have:
\begin{align}
    12e^4H\sqrt{\sum_{h=1}^H\sum_{(s,a)\notin \hat{\mathcal{B}}}p_{h,\mathrm{tar}}^{t+1}(\frac{g_2(\srcn,\delta)}{\srcn}\land 1)}&\le \frac{H}{4}\sqrt{H\frac{\varepsilon^2}{H^3}}\le \varepsilon/4 \nonumber \\
    403e^4H^2\sum_{h=1}^H\sum_{(s,a)\notin \hat{\mathcal{B}}}p_{h,\mathrm{tar}}^{t+1}\frac{g_1(\srcn,\delta)}{\srcn}\land 1&\le H^3\frac{\varepsilon^2}{4H^3}\le \varepsilon^2/4\le \varepsilon/4 \nonumber
\end{align}
For $(s,a)\in\hat{\mathcal{B}}$, we first change the visitation count into the cumulative visitation probability with \cref{lem:occupancy}:
\begin{align}
    12e^{4}H&\sqrt{\sum_{h=1}^H\sum_{(s,a)\in\hat{\mathcal{B}}}p_{h,\mathrm{tar}}^{t+1}(\frac{g_2(\nt,\delta)}{\nt}\land 1)}+403e^{4}H^2\sum_{h=1}^H\sum_{(s,a)\in\hat{\mathcal{B}}}p_{h,\mathrm{tar}}^{t+1}\frac{g_1(\nt,\delta)}{\nt}\land 1 \nonumber \\
    &\le 24e^{4}H\sqrt{\sum_{h=1}^H\sum_{(s,a)\in\hat{\mathcal{B}}}p_{h,\mathrm{tar}}^{t+1}\frac{g_2(\tarovernt,\delta)}{\tarovernt\lor 1}}+1612e^{4}H^2\sum_{h=1}^H\sum_{(s,a)\in\hat{\mathcal{B}}}p_{h,\mathrm{tar}}^{t+1}\frac{g_1(\tarovernt,\delta)}{\tarovernt \lor 1} \nonumber
\end{align}
Notice that $\sum_{h=1}^Hp_{h,\mathrm{tar}}^{t+1}(s,a)=\overline{n}^{t+1}(s,a)-\tarovernt$, thus we obtain:
\begin{align}
    \sum_{h=1}^H\sum_{(s,a)\in\hat{\mathcal{B}}}p_{h,\mathrm{tar}}^{t+1}\frac{g_2(\tarovernt,\delta)}{\tarovernt \lor 1}&=\sum_{(s,a)\in\hat{\mathcal{B}}}\frac{g_2(\tarovernt,\delta)(\overline{n}_{\mathrm{tar}}^{t+1}(s,a)-\tarovernt)}{\tarovernt \lor 1} \nonumber \\
    \sum_{h=1}^H\sum_{(s,a)\in\hat{\mathcal{B}}}p_{h,\mathrm{tar}}^{t+1}\frac{g_1(\tarovernt,\delta)}{\tarovernt \lor 1}&=\sum_{(s,a)\in\hat{\mathcal{B}}}\frac{g_1(\tarovernt,\delta)(\overline{n}_{\mathrm{tar}}^{t+1}(s,a)-\tarovernt)}{\tarovernt \lor 1} \nonumber 
\end{align}
\paragraph{Solving for the sample complexity.}
With the stopping criterion, we notice that for all $t< \tau$, we have:
\begin{align*}
    \varepsilon/2&\le 24 e^{4}H\sqrt{\sum_{(s,a)\in\hat{\mathcal{B}}}\frac{g_2(\tarovernt,\delta)(\overline{n}_{\mathrm{tar}}^{t+1}-\overline{n}_{\mathrm{tar}}^t)(s,a)}{\tarovernt \lor 1}}\\
    &\quad+1612e^{4}H^2\sum_{(s,a)\in\hat{\mathcal{B}}}\frac{g_1(\tarovernt,\delta)(\overline{n}_{\mathrm{tar}}^{t+1}-\overline{n}_{\mathrm{tar}}^t)(s,a)}{\tarovernt \lor 1} 
\end{align*}
Summing over $t<\tau$, we get:
\begin{align}
    (\tau-1)\varepsilon/2&\le 24 e^{4}H\sum_{t=1}^{\tau-1}\sqrt{\sum_{(s,a)\in\hat{\mathcal{B}}}\frac{g_2(\tarovernt,\delta)(\overline{n}_{\mathrm{tar}}^{t+1}-\overline{n}_{\mathrm{tar}}^t)(s,a)}{\tarovernt \lor 1}}\nonumber\\
    &\qquad+1612e^{4}H^2\sum_{t=1}^{\tau-1}\sum_{(s,a)\in\hat{\mathcal{B}}}\frac{g_1(\tarovernt,\delta)(\overline{n}_{\mathrm{tar}}^{t+1}-\overline{n}_{\mathrm{tar}}^t)(s,a)}{\tarovernt \lor 1} \nonumber \\
    &\le 24 e^{4}H\sqrt{\tau-1}\sqrt{\sum_{t=1}^{\tau-1}\sum_{(s,a)\in\hat{\mathcal{B}}}\frac{g_2(\tarovernt,\delta)(\overline{n}_{\mathrm{tar}}^{t+1}-\overline{n}_{\mathrm{tar}}^t)(s,a)}{\tarovernt \lor 1}} \nonumber \\
    &\qquad +1612e^{4}H^2\sum_{t=1}^{\tau-1}\sum_{(s,a)\in\hat{\mathcal{B}}}\frac{g_1(\tarovernt,\delta)(\overline{n}_{\mathrm{tar}}^{t+1}-\overline{n}_{\mathrm{tar}}^t)(s,a)}{\tarovernt \lor 1} \nonumber \tag{Cauchy-Schwarz inequality}\\
    &\le 24 e^{4}H\sqrt{\tau-1}\sqrt{\sum_{t=1}^{\tau-1}\sum_{(s,a)\in\hat{\mathcal{B}}}\frac{g_2(\tau-1,\delta)(\overline{n}_{\mathrm{tar}}^{t+1}-\overline{n}_{\mathrm{tar}}^t)(s,a)}{\tarovernt \lor 1}} \nonumber \\
    &\qquad +1612e^{4}H^2\sum_{t=1}^{\tau-1}\sum_{(s,a)\in\hat{\mathcal{B}}}\frac{g_1(\tau-1,\delta)(\overline{n}_{\mathrm{tar}}^{t+1}-\overline{n}_{\mathrm{tar}}^t)(s,a)}{\tarovernt \lor 1} \nonumber \tag{$g_1(n,\delta)$, $g_2(n,\delta)$ is increasing in $n$}\\
    &\le 48 e^{4}H\sqrt{|\hat{\mathcal{B}}|(\tau-1)g_2(\tau-1,\delta)\log(\tau+1)}+6448e^{4}H^2|\hat{\mathcal{B}}|g_1(\tau-1,\delta)\log(\tau+1) \nonumber \tag{\cref{lem:summation}} \\
    &\le 48 e^{4}H\sqrt{|\hat{\mathcal{B}}|\tau(\log(\frac{6SAH}{\delta})\log(8e\tau)+\log^2(8e\tau))}\nonumber\\
    &\qquad+6448e^{4}H^2|\hat{\mathcal{B}}|(\log(\frac{6HSA}{\delta})\log(8e\tau)+S\log^2(8e\tau)) \nonumber \tag{$\log(\tau+1)\le\log(8e\tau)$}
\end{align}
Solving this inequality by \cref{lem:solveineq}, we obtain that:
\begin{equation}
    \tau\le C_2\frac{H^2|\hat{\mathcal{B}}|}{\varepsilon^2},
\end{equation}
where $C_2>0$ contains only log factors.
By \cref{lem:sample-recognization}, when $\sigma\beta\le \sqrt{\frac{S}{H}\varepsilon}$, the number of total samples from $\mathcal{M}_{\mathrm{tar}}$ in \cref{alg:hybrid} to get a $\varepsilon$-optimal policy for $\mathcal{M}_{\mathrm{tar}}$ is:
\begin{align*}
            \widetilde{O}(\frac{H^2S^2A}{(\sigma\beta)^2}+\frac{H^3|\mathcal{B}|}{\varepsilon^2}).
\end{align*}
\paragraph{Case 2: $\sigma\beta\le \sqrt{\frac{S}{H}\varepsilon}$.}In this case, because we choose $\hat{\mathcal{B}}=\sa$, due to the above proof, the total sample complexity is:
\begin{align*}
    \widetilde{O}(\frac{H^3|\hat{\mathcal{B}}|}{\varepsilon^2})=\widetilde{O}(\frac{H^3SA}{\varepsilon^2}).
\end{align*}
Combining the two cases ends the proof.
\section{Auxiliary Proofs}\label{appendix:auxiliary}
\subsection{Proof of \cref{lem:error-W}}\label{appendix:errow-W}
\begin{proof}
In this lemma, since only $\mathcal{M}_{\mathrm{tar}}$ is considered, we drop the subscript "tar" for clarity. From our construction of reward class $\mathcal{R}$ we can deduce that $0\le V_h^\pi\le 1$. For $\nt>0$, we have:
\begin{align}
    e_h^{\pi,t}(s,a,r)&\triangleq |Q_h^\pi(s,a,r)-Q_h^{\pi,\hat{p}^t}(s,a,r)| \nonumber\\
    &\le |(p-\hat{p}^t)V_{h+1}^\pi|+\hat{p}^t|V_{h+1}^\pi-\hat{V}_{h+1}^{\pi,t}| +|(p-\hat{p}^t)r_h|(s,a)\nonumber\\
    &\le |(p-\hat{p}^t)V_{h+1}^\pi|+\hat{p}^t|V_{h+1}^\pi-\hat{V}_{h+1}^{\pi,t}| +\sqrt{\frac{g_1(\nt,\delta)}{2\nt}}\tag{$|(p-\hat{p}^t)r_h|(s,a)\le\dtv{\hat{p}^t}{p}\le \sqrt{\frac{1}{2}\mathrm{KL}(p,\hat{p}^t)}$}\nonumber\\
    &\le \sqrt{2\revar(V_{h+1}^\pi)(s,a,r)\frac{g_1(\nt,\delta)}{\nt}}+\frac{2}{3}\frac{g_1(\nt,\delta)}{\nt}+\hat{p}^t|V_{h+1}^\pi-\hat{V}_{h+1}^{\pi,t}|+\sqrt{\frac{g_1(\nt,\delta)}{2\nt}} \tag{\cref{lem:trans}} \\
    &\le 3\sqrt{\frac{g_1(\nt,\delta)}{\nt}}+\frac{2}{3}\frac{g_1(\nt,\delta)}{\nt}+\hat{p}^t|V_{h+1}^\pi-\hat{V}_{h+1}^{\pi,t}| \nonumber\\
    &\le 3\sqrt{\frac{1}{H}(\frac{Hg_1(\nt,\delta)}{\nt}\land \frac{1}{H})}+4H\frac{g_1(\nt,\delta)}{\nt}+\hat{p}^t\pi_{h+1}e_{h+1}^{\pi,t}(s,a,r), \nonumber
\end{align}
where in the last inequality we utilize that if $\frac{Hg_1(\nt,\delta)}{\nt}\ge\frac{1}{H}$, we have:
\begin{align*}
    \sqrt{\frac{g_1(\nt,\delta)}{\nt}}=\sqrt{\frac{1}{H^2}H^2\frac{g_1(\nt,\delta)}{\nt}}\le \frac{1}{H} H^2\frac{g_1(\nt,\delta)}{\nt}=H\frac{g_1(\nt,\delta)}{\nt}.
\end{align*}
Motivated by \citet{menard_fast_2021}, we further define $W_{h}^{\pi,t}$ recursively by $W_{H+1}^{\pi,t}(s,a)=0$ and:
\begin{equation*}
    W_h^{\pi,t}(s,a)\triangleq \min\left(1,\frac{4Hg_1(\nt,\delta)}{\nt} +\hat{p}^{t}\pi_{h+1}W_{h+1}^{\pi,t}(s,a)\right).
\end{equation*}
Besides, we define $Y_h^{\pi,t}(s,a)$ recursively by $Y_{H+1}^{\pi,t}(s,a)=0$ and:
\begin{equation*}
    Y_h^{\pi,t}(s,a,r)\triangleq 3\sqrt{\frac{1}{H}(\frac{Hg_1(\nt,\delta)}{\nt}\land \frac{1}{H})}+\hat{p}^t\pi_{h+1}Y_{h+1}^{\pi,t}(s,a).
\end{equation*}
With the above definitions, we can show by induction that $e_h^{\pi,t}(s,a,r)\le Y_h^{\pi,t}(s,a)+W_h^{\pi,t}(s,a)$. For the base case, the claim trivially holds. Assume the argument holds for step $h+1$, we have:
\begin{align*}
    e_h^{\pi,t}(s,a,r)&\le3\sqrt{\frac{1}{H}(\frac{Hg_1(\nt,\delta)}{\nt}\land \frac{1}{H})}+4H\frac{g_1(\nt,\delta)}{\nt}+\hat{p}^t\pi_{h+1}(Y_{h+1}^{\pi,t}(s,a)+W_{h+1}^{\pi,t}(s,a)). \tag{induction hypothesis}
\end{align*}
Because $e_h^{\pi,t}(s,a,r)\le 1$, we actually have:
\begin{align}
    e_h^{\pi,t}(s,a,r)&\le 3\sqrt{\frac{1}{H}(\frac{Hg_1(\nt,\delta)}{\nt}\land \frac{1}{H})}+W_h^{\pi,t}(s,a)+\hat{p}^t\pi_{h+1}Y_{h+1}^{\pi,t}(s,a)\nonumber\\
    &\le Y_h^{\pi,t}(s,a)+W_h^{\pi,t}(s,a).\label{eq:error-decomposition}
\end{align}
In the following step, the core idea is to obtain an upper bound on $[\rho \pi_1  e_1^{\pi,t}](r)$. By applying \cref{eq:error-decomposition}, we obtain that:
\begin{align}
    [\rho \pi_1  e_1^{\pi,t}](r)&\le \rho \pi_1  Y_1^{\pi,t}+\rho \pi_1  W_1^{\pi,t} \nonumber\\
    &\le 3\sum_{h=1}^H\sum_{(s,a)\in\sa}\hat{p}_h^{\pi,t}\sqrt{1/H(\frac{Hg_1(\nt,\delta)}{\nt} \land \frac{1}{H})} +\rho \pi_1  W_1^{\pi,t}\nonumber\\
    &\le 3\sqrt{\sum_{h=1}^H\sum_{(s,a)\in\sa}\hat{p}_h^{\pi,t}(s,a)/H}\sqrt{\sum_{h=1}^H\sum_{(s,a)\in\sa}\hat{p}_h^{\pi,t}(\frac{Hg_1(\nt,\delta)}{\nt}\land \frac{1}{H}) } +\rho \pi_1  W_1^{\pi,t} \tag{Cauchy-Schwarz inequality} \nonumber\\
    &= 3\sqrt{\sum_{h=1}^H\sum_{(s,a)\in\sa}\hat{p}_h^{\pi,t}(\frac{Hg_1(\nt,\delta)}{\nt} \land \frac{1}{H})}+\rho \pi_1  W_1^{\pi,t}.\label{eq:uncertainty-bound}
\end{align}
We recursively define $\hat{W}_h^{\pi,t}(s,a)$ by $\hat{W}_{H+1}^{\pi,t}(s,a)\triangleq 0$ and:
\begin{equation*}
    \hat{W}_h^{\pi,t}(s,a)\triangleq \frac{Hg_1(\nt,\delta)}{\nt} \land \frac{1}{H}+\hat{p}^{t}\pi_{h+1}\hat{W}_{h+1}^{\pi,t}(s,a).
\end{equation*}
By definition, we know that:
\begin{equation*}
    \rho \pi_1 \hat{W}_1^{\pi,t}=\sum_{h=1}^H\sum_{(s,a)\in\sa}\hat{p}_h^{\pi,t}\frac{Hg_1(\nt,\delta)}{\nt}\land \frac{1}{H}.
\end{equation*}
Thus \cref{eq:uncertainty-bound} can be rewritten as:
\begin{equation*}
    [\rho \pi_1  e_1^{\pi,t}](r)\le 3\sqrt{\rho \pi_1 \hat{W}_1^{\pi,t}}+\rho \pi_1  W_1^{\pi,t}.
\end{equation*}
We now show that $\hat{W}_h^{\pi,t}(s,a)\le W_{h}^{\pi,t}(s,a)\le W_h^t(s,a)$. For the base case, the claim trivially holds since $\hat{W}_{H+1}^{\pi,t}(s,a)= W_{H+1}^{\pi,t}(s,a)= W_{H+1}^t(s,a)=0$. Assume the claim holds for step $h+1$, and for step $h$ we have:
\begin{align*}
    W_h^{\pi,t}(s,a)&\le \min\left(1,\frac{4Hg_1(\nt,\delta)}{\nt}+\hat{p}^{t}\pi_{h+1}W_{h+1}^{t}(s,a)\right) \tag{induction hypothesis} \\
    &\le \min\left(1,\frac{4Hg_1(\nt,\delta)}{\nt} +\hat{p}^{t}\max_{a'\in\mathcal{A}}W_{h+1}^{t}(s,a)\right) \\
    &=W_h^t(s,a).
\end{align*}
Since by our construction, we have $\hat{W}_h^{\pi,t}(s,a)\le 1$, thus:
\begin{align*}
    \hat{W}_h^{\pi,t}(s,a)&\le \min\left(1,\frac{Hg_1(\nt,\delta)}{\nt} \land \frac{1}{H}+\hat{p}^{t}\pi_{h+1}\hat{W}_{h+1}^{\pi,t}(s,a)\right) \\
    &\le \min\left(1,\frac{4Hg_1(\nt,\delta)}{\nt}+\hat{p}^{t}\pi_{h+1}W_{h+1}^{\pi,t}(s,a)\right) \tag{induction hypothesis}\\
    &=W_h^{\pi,t}(s,a).
\end{align*}
Therefore by induction, we have:
\begin{equation*}
    \rho \pi_1 \hat{W}_1^{\pi,t}\le\rho \pi_1  W_1^{\pi,t}\le \rho \pi_1  W_1^{t}.
\end{equation*}
So we obtain that:
\begin{equation}
    [\rho \pi_1  e_1^{\pi,t}](r)\le 3\sqrt{\rho \pi_1  W_1^{t}}+\rho \pi_1  W_1^{t}. \label{eq:error-upper-bound}
\end{equation}
\end{proof}
\subsection{Proof of \cref{lem:optimism}}\label{appendix:optimism}
\begin{proof}
We proceed by induction. For the base case, we have $\overline{Q}_{H+1}^t(s,a)=Q_{H+1}^{p_{\mathrm{tar}},\star}(s,a)=\underline{Q}_{H+1}^t=0$, so \cref{eq:optimism-q} and \cref{eq:optimism-v} trivially hold. Assuming the statement hold for step $h+1$, then for step $h$, we have:
\begin{align}
    \overline{Q}_{h}^t(s,a)-Q_{h}^{p_{\mathrm{tar}},\star}(s,a)&=3\sqrt{\tmvar(\overline{V}_{h+1}^t)(s,a)\frac{g_2(\tnt(s,a),\delta)}{\tnt(s,a)}}+14H^2\frac{g_1(\tnt(s,a),\delta)}{\tnt(s,a)} \nonumber \\
    &\qquad+\tpt(\overline{V}_{h+1}^t-\underline{V}_{h+1}^t)(s,a)+\tpt(\overline{V}_{h+1}^t-V_{h+1}^{p_{\mathrm{tar}},\star})(s,a)+(\tpt-p_{\mathrm{tar}})V_{h+1}^{p_{\mathrm{tar}},\star}(s,a). \nonumber
\end{align}
Then by the definition of $F_4$, we have:
\begin{equation*}
    \left| (\tpt - p_{\mathrm{tar}}) V_{h+1}^{p_{\mathrm{tar}},\star}(s, a) \right| \leq\sqrt{2 \mathrm{Var}_{p_{\mathrm{tar}}}(V_{h+1}^{p_{\mathrm{tar}},\star})(s, a) \frac{g_2(\tnt(s, a), \delta)}{\tnt(s, a)}} + \frac{3H g_2(\tnt(s, a), \delta)}{\tnt(s, a)},
\end{equation*}
then by \cref{lem:var-pq} and \cref{lem:var-fg}, we can substitute the variance of optimal value function under  the true dynamics with the one under the empirical dynamics,
\begin{align}
    \mathrm{Var}_{p_{\mathrm{tar}}}(V_{h+1}^{p_{\mathrm{tar}},\star})&\le2\tmvar(V_{h+1}^{p_{\mathrm{tar}},\star})+4H^2\frac{g_1(\tnt(s,a),\delta)}{\tnt(s,a)} \nonumber \\
    &\le 4\tmvar(\overline{V}_{h+1})+4H^2\frac{g_1(\tnt(s,a),\delta)}{\tnt(s,a)}+4H\tpt(\overline{V}_{h+1}^t-V_{h+1}^{p_{\mathrm{tar}},\star}) \nonumber \\
    &\le 4\tmvar(\overline{V}_{h+1})+4H^2\frac{g_1(\tnt(s,a),\delta)}{\tnt(s,a)}+4H\tpt(\overline{V}_{h+1}^t-\underline{V}_{h+1}^t). \tag{Induction hypothesis} \nonumber
\end{align}
Notice that $g_1(n,\delta)\ge g_2(n,\delta)$, and by inequalities $\sqrt{x+y}\le\sqrt{x}+\sqrt{y}$ and $\sqrt{xy}\le x+y$ we have:
\begin{align}
    \left| (\tpt - p) V_{h+1}^{p_{\mathrm{tar}},\star}(s, a) \right|&\le\sqrt{8\tmvar(\overline{V}_{h+1}^t)\frac{g_2(\tnt(s,a),\delta)}{\tnt(s,a)}}+(3+2\sqrt{2})H\frac{g_1(\tnt(s,a),\delta)}{\tnt(s,a)} \nonumber\\
    &\qquad+\sqrt{8H^2\frac{g_1(\tnt(s,a),\delta)}{\tnt(s,a)}\frac{1}{H}\tpt(\overline{V}_{h+1}^t-\underline{V}_{h+1}^t)} \nonumber \\
    &\le 3\sqrt{\tmvar(\overline{V}_{h+1}^t)\frac{g_2(\tnt(s,a),\delta)}{\tnt(s,a)}}+14H^2\frac{g_1(\tnt(s,a),\delta)}{\tnt(s,a)}+\frac{1}{H}\tpt(\overline{V}_{h+1}^t-\underline{V}_{h+1}^t). \nonumber
\end{align}
Then with the induction hypothesis, we have:
\begin{equation*}
    \overline{Q}_{h}^t(s,a)-Q_{h}^{p_{\mathrm{tar}},\star}(s,a)\ge\tpt(\overline{V}_{h+1}^t-\underline{V}_{h+1}^t)(s,a)\ge 0
\end{equation*}
The proof for $\underline{Q}_{h}^t(s,a)\le Q_{h}^{p_{\mathrm{tar}},\star}(s,a)$ follows a similar procedure and we omit this proof for brevity. Then by definition, we also have:
\begin{align}
    \overline{V}_{h}^t(s)&=\max_{a\in\mathcal{A}}\overline{Q}_{h}^t(s,a)\ge \max_{a\in\mathcal{A}}Q_h^{p_{\mathrm{tar}},\star}(s,a)=V_h^{p_{\mathrm{tar}},\star}(s), \nonumber \\
    \underline{V}_{h}^t(s)&=\max_{a\in\mathcal{A}}\underline{Q}_{h}^t(s,a)\le \max_{a\in\mathcal{A}}Q_h^{p_{\mathrm{tar}},\star}(s,a)=V_h^{p_{\mathrm{tar}},\star}(s). \nonumber
\end{align}
This end the proof.
\end{proof}
\subsection{Proof of \cref{lem:value-gap-bound}}\label{appendix:value-gap-bound}
\begin{proof}
    By \cref{lem:lower-estimate}, we only need to show $\overline{Q}_h^t(s,a)-\tilde{Q}_h^t(s,a)\le G_h^t(s,a)$, and we proceed by induction. For the base case, the claim trivially holds because $\overline{Q}_{H+1}^t(s,a)=\tilde{Q}_{H+1}^t(s,a)=G_h^t(s,a)=0$. Assuming the claim holds for step $h+1$, then for step $h$, if $G_h^t(s,a)=H$, the claim still trivially holds. If $G_h^t(s,a)\neq H$, we discuss two cases.
    \paragraph{Case 1} $\tilde{Q}_h^t(s,a)=r(s,a)+p_{\mathrm{tar}}\tilde{V}_{h+1}^t(s,a)$, then we have:
    \begin{align}
        \overline{Q}_h^t(s,a)-\tilde{Q}_h^t(s,a)&\le 3 \sqrt{\mathrm{Var}_{\tpt} (\overline{V}_{h+1}^t)(s, a) \frac{g_2(\tnt(s, a), \delta)}{\tnt(s, a)}} + 14H^2 \frac{g_1(\tnt(s, a), \delta)}{\tnt(s, a)} \nonumber \\
    &\qquad + \frac{1}{H} \tpt (\overline{V}_{h+1}^t - \underline{V}_{h+1}^t)(s, a) +(\tpt-p_{\mathrm{tar}})V_{h+1}^{p_{\mathrm{tar}},*}(s,a)\nonumber \\
    &\qquad+\hat{p}^t(\overline{V}_{h+1}^t-\tilde{V}_{h+1}^t)(s,a)+(p_{\mathrm{tar}}-\tpt)(V_{h+1}^{p_{\mathrm{tar}},\star}-\tilde{V}_{h+1}^t)(s,a). \nonumber
    \end{align}
    From the proof of \cref{lem:optimism}, we have:
    \begin{equation*}
        \left| (\tpt - p_{\mathrm{tar}}) V_{h+1}^{p_{\mathrm{tar}},\star}(s, a) \right|\le 3\sqrt{\tmvar(\overline{V}_{h+1}^t)\frac{g_2(\tnt(s,a),\delta)}{\tnt(s,a)}}+14H^2\frac{g_1(\tnt(s,a),\delta)}{\tnt(s,a)}+\frac{1}{H}\tpt(\overline{V}_{h+1}^t-\underline{V}_{h+1}^t). \nonumber
    \end{equation*}
    By \cref{lem:trans} and inequality $\sqrt{xy}\le x+y$ we have:
    \begin{align}
        |(p_{\mathrm{tar}}-\tpt)(V_{h+1}^{p_{\mathrm{tar}},*}-\tilde{V}_{h+1}^t)(s,a)|&\le\sqrt{\frac{2}{H^2}\tarvar(V_{h+1}^{p_{\mathrm{tar}},*}-\tilde{V}_{h+1}^t)(s,a)H^2\frac{g_1(\tnt(s,a),\delta)}{\tnt(s,a)}}+\frac{2}{3}H\frac{g_1(\tnt(s,a),\delta)}{\tnt(s,a)} \nonumber \\
        &\le \frac{1}{2H^2}\tarvar(V_{h+1}^{p_{\mathrm{tar}},*}-\tilde{V}_{h+1}^t)(s,a)+5H^2\frac{g_1(\tnt(s,a),\delta)}{\tnt(s,a)} \nonumber \\
        &\le \frac{1}{H^2}\tmvar(V_{h+1}^{p_{\mathrm{tar}},*}-\tilde{V}_{h+1}^t)(s,a)+7H^2\frac{g_1(\tnt(s,a),\delta)}{\tnt(s,a)} \nonumber  \tag{\cref{lem:var-pq}}\\
        &\le \frac{1}{H}\tpt(V_{h+1}^{p_{\mathrm{tar}},*}-\tilde{V}_{h+1}^t)(s,a)+7H^2\frac{g_1(\tnt(s,a),\delta)}{\tnt(s,a)} \nonumber \\
        &\le\frac{1}{H}\tpt(V_{h+1}^{p_{\mathrm{tar}},*}-\tilde{V}_{h+1}^t)(s,a)+7H^2\frac{g_1(\tnt(s,a),\delta)}{\tnt(s,a)}, \nonumber \tag{\cref{lem:optimism}}
    \end{align}
    then we have:
    \begin{align}
        \overline{Q}_h^t(s,a)-\tilde{Q}_h^t(s,a)&\le 6\sqrt{\mathrm{Var}_{\tpt} (\overline{V}_{h+1}^t)(s, a) \frac{g_2(\tnt(s, a), \delta)}{\tnt(s, a)}}+35H^2\frac{g_1(\tnt(s,a),\delta)}{\tnt(s,a)}+\frac{2}{H}\tpt (\overline{V}_{h+1}^t - \underline{V}_{h+1}^t)(s, a) \nonumber \\
        &\qquad+\frac{1}{H}\tpt(V_{h+1}^{p_{\mathrm{tar}},*}-\tilde{V}_{h+1}^t)(s,a)+\tpt(\overline{V}_{h+1}^t-\tilde{V}_{h+1}^t)(s,a) \nonumber \\
        &\le 6\sqrt{\mathrm{Var}_{\tpt} (\overline{V}_{h+1}^t)(s, a) \frac{g_2(\tnt(s, a), \delta)}{\tnt(s, a)}}+35H^2\frac{g_1(\tnt(s,a),\delta)}{\tnt(s,a)} \nonumber\\
        &\qquad+(1+\frac{3}{H})\tpt (\overline{V}_{h+1}^t - \tilde{V}_{h+1}^t)(s, a) \nonumber \tag{\cref{lem:lower-estimate}} \\
        &\le 6\sqrt{\mathrm{Var}_{\tpt} (\overline{V}_{h+1}^t)(s, a) \frac{g_2(\tnt(s, a), \delta)}{\tnt(s, a)}}+35H^2\frac{g_1(\tnt(s,a),\delta)}{\tnt(s,a)} \nonumber\\
        &\qquad+(1+\frac{3}{H})\tpt\pi_{h+1}^{t+1}G_{h+1}^t(s,a). \nonumber \tag{Induction hypothesis} \\ \nonumber
        &=G_h^t(s,a)
    \end{align}
    \paragraph{Case 2} $\tilde{Q}_h^t(s,a)\neq r(s,a)+p_{\mathrm{tar}}\tilde{V}_{h+1}^t(s,a)$, then we have:
    \begin{align}
        \overline{Q}_h^t(s,a)-\tilde{Q}_h^t(s,a)&\le 6\sqrt{\mathrm{Var}_{\tpt} (\overline{V}_{h+1}^t)(s, a) \frac{g_2(\tnt(s, a), \delta)}{\tnt(s, a)}}+28H^2\frac{g_1(\tnt(s,a),\delta)}{\tnt(s,a)} \nonumber \\
        &\qquad +\frac{2}{H}\tpt (\overline{V}_{h+1}^t - \underline{V}_{h+1}^t)(s, a)+\tpt (\overline{V}_{h+1}^t - \tilde{V}_{h+1}^t)(s, a) \nonumber \\
        &\le 6\sqrt{\mathrm{Var}_{\tpt} (\overline{V}_{h+1}^t)(s, a) \frac{g_2(\tnt(s, a), \delta)}{\tnt(s, a)}}+35H^2\frac{g_1(\tnt(s,a),\delta)}{\tnt(s,a)} \nonumber \\
        &\qquad (1+\frac{3}{H})\tpt (\overline{V}_{h+1}^t - \tilde{V}_{h+1}^t)(s, a) \nonumber \\
        & \le 6\sqrt{\mathrm{Var}_{\tpt} (\overline{V}_{h+1}^t)(s, a) \frac{g_2(\tnt(s, a), \delta)}{\tnt(s, a)}}+35H^2\frac{g_1(\tnt(s,a),\delta)}{\tnt(s,a)} \nonumber \tag{\cref{lem:lower-estimate}}\\
        &\qquad (1+\frac{3}{H})\tpt\pi_{h+1}^{t+1}G_{h+1}^t(s,a) \nonumber \tag{Induction hypothesis}\\
        &=G_h^t(s,a). \nonumber
    \end{align}
    Thus, we have:
    \begin{align}
        V_1^{p_{\mathrm{tar}},\star}(\rho)-V_1^{p_{\mathrm{tar}},\pi^{t+1}}(\rho)&\le \mathbb{E}_{s\sim \rho}(V_1^{p_{\mathrm{tar}},*}(s)-V_1^{p_{\mathrm{tar}},\pi^{t+1}}(s)) \nonumber \\
        &\le \mathbb{E}_{s\sim \rho}(\overline{V}_1^{t}(s)-\tilde{V}_1^{t}(s)) \nonumber  \tag{\cref{lem:optimism} and \cref{lem:lower-estimate}}\\
        &\le \rho\pi_1^{t+1}(\overline{Q}_1^{t}-\tilde{Q}_1^{t}) \nonumber \\
        &\le \rho\pi_1^{t+1}G_1^t. \nonumber 
    \end{align}
\end{proof}
\subsection{Auxiliary Lemmas}
\begin{lemma}[Visitation count to cumulative visitation probability]
\label{lem:occupancy}
On event $F^{\mathrm{RF}}$, $\forall (s, a) \in \sa$, 
\begin{equation*}
    \forall t \in \mathbb{N}^+, \quad \frac{g_1(\nt, \delta)}{\nt} \land 1 \leq 4 \frac{g_1(\tarovernt, \delta)}{\tarovernt\lor 1} .
\end{equation*}
\begin{equation*}
    \forall t \in \mathbb{N}^+, \quad \frac{g_2(\nt, \delta)}{\nt} \land 1 \leq 4 \frac{g_2(\tarovernt, \delta)}{\tarovernt\lor 1} .
\end{equation*}
On event $F^{\mathrm{Hybrid}}$, $\forall (s, a) \in \hat{\mathcal{B}}$, 
\begin{equation*}
    \forall t \in \mathbb{N}^+, \quad \frac{g_1(\tnt(s,a), \delta)}{\tnt(s,a)} \land 1 \leq 4 \frac{g_1(\tarovernt, \delta)}{\tarovernt\lor 1} .
\end{equation*}
\begin{equation*}
    \forall t \in \mathbb{N}^+, \quad \frac{g_2(\tnt(s, a), \delta)}{\tnt(s, a)} \land 1 \leq 4 \frac{g_2(\tarovernt, \delta)}{\tarovernt\lor 1} .
\end{equation*}
\end{lemma}
\begin{proof}
    We only prove the result for $g_1$, and the proof for $g_2$ is exactly the same. On event $F^{\mathrm{RF}}$, $\forall (s,a)\in \sa$, $\forall t\in \mathbb{N}^+$, we have
    \begin{equation*}
        \nt \geq \frac{1}{2} \tarovernt - g_3(\delta)
    \end{equation*}
    \paragraph{Case 1} If $g_3(\delta)\le \frac{1}{4}\tarovernt$, then
    \begin{equation*}
        \frac{g_1(\nt,\delta)}{\nt}\land 1\le \frac{g_1(\nt,\delta)}{\nt}\le \frac{g_1(\frac{1}{4}\tarovernt,\delta)}{\frac{1}{4}\tarovernt}\le 4\frac{g_1(\tarovernt,\delta)}{\tarovernt\lor 1}
    \end{equation*}
    Here the second inequality is due to that $\frac{g_1(x,\delta)}{x}$ is non-increasing for $x\ge 1$, and the third inequality is due to $g_1(x,\delta)$ is non-decreasing and $\tarovernt\ge 4g_3(\delta)\ge 1$.
    \paragraph{Case 2} If $g_3(\delta)> \frac{1}{4}\tarovernt$, similarly we can get:
    \begin{equation*}
        \frac{g_1(\nt,\delta)}{\nt}\land 1\le 1<4\frac{g_3(\delta)}{\tarovernt \lor 1}\le 4\frac{g_1(\tarovernt, \delta)}{\tarovernt\lor 1}
    \end{equation*}
    where we utilize the fact that $g_1(x,\delta)\ge g_3(\delta)\ge 1$ always holds for $x\ge 0$. Combine the two cases and we get the results. Similarly, on $F^{\mathrm{Hybrid}}$, because we also have:
    \begin{equation*}
        \tnt(s,a) \geq \frac{1}{2} \tarovernt - g_3(\delta)
    \end{equation*}
    So the same claim holds for $F^{\mathrm{Hybrid}}$. This ends the proof.
\end{proof}

\begin{lemma}[Concentration for Discrete Distributions, \citep{hsu2012spectralalgorithmlearninghidden}]\label{lem:concentration-discrete}
Let $z$ be a discrete random variable that takes values in $\{1,\cdots,d\}$, distributed according to p. We write $p$ as a vector where $q=[\mathrm{Pr}(z=j)]_{j=1}^d$. Assume we have $N$ iid samples, and our empirical estimate of $q$ is $[\hat{q}]_j=\sum_{i=1}^N\indic{z_i=j}/N$. We have for $\forall\varepsilon>0$:
\begin{equation*}
    \mathrm{Pr}\left(||\hat{q}-q||_{1}\ge \sqrt{d}(1/\sqrt{N}+\varepsilon)\right)\le e^{-N\varepsilon^2}.
\end{equation*}   
\end{lemma}
Next, we introduce the deviation inequality for categorical distributions by \citet[Proposition 1]{jonsson2020planning}. Note that this result helps for the concentration of the events $F_1, F_4$-- which requires uniform control over all $t\in \mathbb{N}^+$,

\begin{lemma}[\text{\citet[Proposition 1]{jonsson2020planning}}]\label{lem:kl-divergence}
For all $ p \in \Sigma_m $ and for all $ \delta \in [0,1] $,
\begin{equation*}
    \mathbb{P} \left( \forall n \in \mathbb{N}^+, n \mathrm{KL}(\hat{p}_n, p) > \log(\frac{1}{\delta}) + (m-1) \log \left( e(1 + \frac{n}{m-1}) \right) \right) \ge 1-\delta.
\end{equation*}
\end{lemma}
\begin{lemma}[Concentration for Bernoulli sequences]\label{lem:bernoulli}
Let $\{\mathcal{F}_i\}_{i=1}^n $ be a filtration and $ X_1, \ldots, X_n $ be a sequence of Bernoulli random variables with $ \mathbb{P}(X_i = 1 \mid \mathcal{F}_{i-1}) = P_i $ with $ P_i $ being $ \mathcal{F}_{i-1} $-measurable and $ X_i $ being $ \mathcal{F}_i $-measurable. It holds that
\begin{equation*}
    \mathbb{P} \left( \forall n, \sum_{t=1}^n X_t > \sum_{t=1}^n P_t / 2 -\log(\frac{1}{\delta})  \right) \ge 1-\delta
\end{equation*}
\end{lemma}
\begin{proof}
We first define the following sequence:
\begin{equation*}
M_n = e^{\sum_{t=1}^n (-X_t + P_t / 2)}
\end{equation*}
which is a supermartingale because (here we define $f(p)=e^{p/2}(1-p+\frac{p}{e})$, $p\in [0,1]$):
\begin{align}
    \mathbb{E}(M_n\mid \mathcal{F}_{n-1})&=e^{\sum_{t=1}^{n-1} (-X_t+ P_t/2)}\mathbb{E}(e^{P_n/2-X_n}\mid \mathcal{F}_{n-1}) \nonumber \\
    &=e^{\sum_{t=1}^{n-1} (-X_t+ P_t/2)}f(P_n) \nonumber \tag{definition of $P_t$} \\
    &\le e^{\sum_{t=1}^{n-1} (-X_t+ P_t/2)}f(0) \nonumber \tag{$f(p)$ is decreasing in $[0,1]$} \\
    &=M_{n-1} \nonumber
\end{align}
With the well-known Ville's inequality \citep[Exercise 4.8.2]{durrett_probability}, for any non-negative  supermartingale $M_n$ we have:
\begin{equation*}
    \mathbb{P}(\sup_{n\in\mathbb{N}_+}M_n>\frac{1}{\delta})\le \delta\cdot \mathbb{E}(M_1).
\end{equation*}
Since $f(x)\triangleq x e^{-1+x/2}+(1-x)e^{x/2}$ is non-increasing in $[0,1]$, we have
\[ E(M_1) = P_1 e^{-1 + P_1/2} + (1-P_1) e^{P_1/2} \le 1.\]
This implies that $\mathbb{P}(\sup_{n\in\mathbb{N}_+}M_n<\frac{1}{\delta})\ge 1-\delta$. Thus, we obtain:
\begin{equation*}
    \mathbb{P} \left( \forall n, \sum_{t=1}^n X_t > \sum_{t=1}^n P_t / 2 -\log(\frac{1}{\delta})  \right) \ge 1-\delta
\end{equation*}
\end{proof}
\begin{lemma}[Bernstein type inequality \citep{domingues2022kernelbased}]\label{lem:bernstein}
Consider the sequences of random variables $(w_t)_{t \in \mathbb{N}^*}$ and $(Y_t)_{t \in \mathbb{N}^*}$ adapted to a filtration $(\mathcal{F}_t)_{t \in \mathbb{N}}$. Let
\begin{equation*}
    S_t \triangleq \sum_{s=1}^t w_s Y_s, \quad V_t \triangleq \sum_{s=1}^t w_s^2 \mathbb{E} \left[ Y_s^2 \mid \mathcal{F}_{s-1} \right], \quad \text{and} \quad W_t \triangleq \sum_{s=1}^t w_s,
\end{equation*}
and $h(x) \triangleq (x + 1) \log(x + 1) - x$. Assume that, for all $t \geq 1$,
\begin{itemize}
    \item $w_t$ is $\mathcal{F}_{t-1}$ measurable,
    \item $\mathbb{E} \left[ Y_t \mid \mathcal{F}_{t-1} \right] = 0$,
    \item $w_t \in [0, 1]$ almost surely,
    \item there exists $b > 0$ such that $|Y_t| \leq b$ almost surely.
\end{itemize}
Then, for all $\delta >0$,
	\begin{align*}
		\mathbb{P}(\exists t\geq 1,   (V_t/b^2+1)h\left(\!\frac{b |S_t|}{V_t+b^2}\right) \geq \log(1/\delta) + \log\left(4e(2t+1)\!\right))\leq \delta.
	\end{align*}
  The previous inequality can be weakened to obtain a more explicit bound: with probability at least $1-\delta$, for all $t\geq 1$,
 \begin{equation*}
      |S_t|\leq \sqrt{2V_t \log\left(4e(2t+1)/\delta\right)}+ 3b\log\left(4e(2t+1)/\delta\right)\,.
 \end{equation*}
\end{lemma}
\begin{lemma}[Bernstein transportation \citep{talebi2018variance}]\label{lem:trans}
Let $ p, q \in \Sigma_S $, where $\Sigma_S$ denotes the probability simplex of dimension $S - 1$. For all $\alpha > 0$, for all functions $f$ defined on $\mathcal{S}$ with $0 \leq f(k) \leq b$, for all $s \in \mathcal{S}$, if $\mathrm{KL}(p, q) \leq \alpha$ then

\begin{equation*}
    |pf - qf| \leq \sqrt{2 \mathrm{Var}_q(f) \alpha} + \frac{2}{3} b \alpha,
\end{equation*}
where we use the expectation operator defined as $ pf \triangleq \mathbb{E}_{s \sim p} f(s) $ and the variance operator defined as $\mathrm{Var}_p(f) \triangleq \mathbb{E}_{s \sim p} (f(s) - \mathbb{E}_{s' \sim p} f(s'))^2 = p(f - pf)^2 $.
\end{lemma}
Next, we prove a standard probability fact for our problem setting.
\begin{lemma}[Law of total variance]\label{lem:tv}
For any policy $\pi$ and for all $h \in [H]$,
\begin{equation*}
    \mathbb{E}_{\pi} \left[
    \left( \sum_{h = 1}^H r(s_h, a_h) - V_1^{\pi}(s_1) \right)^2
\right] = \sum_{h=1}^H \sum_{s,a} p^{\pi}_h(s, a) \mathrm{Var}_{p} \left( V_{h+1}^{\pi} \right)(s, a).
\end{equation*}
\end{lemma}
\begin{proof}
    We intend to prove the following claim:
    \begin{equation*}
        \mathbb{E}_{\pi} \left[
        \left( \sum_{h' = h}^H r(s_h, a_h) - V_h^{\pi}(s_h) \right)^2
    \middle| s_h\right] =\sum_{h'=h}^H \sum_{s_{h'},a_{h'}} [p^{\pi}_h(s_{h'},a_{h'}) \mathrm{Var}_{p} \left( V_{h'+1}^{\pi} \right)(s_{h'},a_{h'})\mid s_h].
    \end{equation*}
    We proceed by induction. For the base case, the claim trivially holds since $V_{H+1}^\pi(s)=0$. We now assume the claim holds for step $h+1$, and for step $h$, we have:
    \begin{align*}
        &\mathbb{E}_{\pi} \left[
        \left( \sum_{h' = h}^H r(s_h, a_h) - V_h^{\pi}(s_h) \right)^2
    \middle| s_h\right] \\
    &=\mathbb{E}_{\pi} \left[
        \left( \sum_{h' = h+1}^H r(s_{h'},a_{h'}) - V_{h+1}^{\pi}(s_{h+1})+V_{h+1}^{\pi}(s_{h+1})-pV_{h+1}^\pi(s_h,a_h) \right)^2
    \middle| s_h\right] \\
    &= \mathbb{E}_{\pi} \left[ \left( V_{h+1}^{\pi}(s_{h+1}) - p V_{h+1}^{\pi}(s_h) \right)^2 \middle| s_h \right] +\mathbb{E}_{\pi} \left[ \left( \sum_{h' = h+1}^{H} r_{h'}(s_{h'}, a_{h'}) - V_{h+1}^{\pi}(s_{h+1}) \right)^2 \middle| s_h \right] \\
    &\quad + 2 \mathbb{E}_{\pi} \left[ \left( \sum_{h' = h+1}^{H} r_{h'}(s_{h'}, a_{h'}) - V_{h+1}^{\pi}(s_{h+1}) \right) \left( V_{h+1}^{\pi}(s_{h+1}) - p V_{h+1}^{\pi}(s_h) \right) \middle| s_h \right]
    \end{align*}
For the cross term, due to the definition of $V_{h+1}^{\pi}$, we have
\begin{equation*}
    \mathbb{E}_{\pi} \left[ \sum_{h' = h+1}^{H} r_{h'}(s_{h'}, a_{h'}) - V_{h+1}^{\pi}(s_{h+1}) \middle| s_{h+1} \right] = 0
\end{equation*}
Therefore by the law of total expectation we know that the cross term equals zero. This shows that:
\begin{align*}
    &\mathbb{E}_{\pi} \left[ \left( \sum_{h' = h}^{H} r_{h'}(s_{h'}, a_{h'}) - V_h^{\pi}(s_h) \right)^2 \middle| s_h \right] \\
    &= \mathbb{E}_{\pi} \left[ \left( V_{h+1}^{\pi}(s_{h+1}) - p V_{h+1}^{\pi}(s_h) \right)^2 \middle| s_h\right] +\mathbb{E}_{\pi} \left[ \left( \sum_{h' = h+1}^{H} r_{h'}(s_{h'}, a_{h'}) - V_{h+1}^{\pi}(s_{h+1}) \right)^2 \middle| s_h \right] \\
    &=\sum_{a_h}[p_h^\pi(s_h,a_h)\revar(V_{h+1}^\pi)(s_h,a_h)\mid s_h]+\sum_{h'=h+1}^H \sum_{s_{h'},a_{h'}} [p^{\pi}_h(s_{h'},a_{h'}) \mathrm{Var}_{p} \left( V_{h'+1}^{\pi} \right)(s_{h'},a_{h'})\mid s_h] \tag{induction hypothesis} \\
    &=\sum_{h'=h}^H \sum_{s_{h'},a_{h'}} [p^{\pi}_h(s_{h'},a_{h'}) \mathrm{Var}_{p} \left( V_{h'+1}^{\pi} \right)(s_{h'},a_{h'})\mid s_h] 
\end{align*}
This ends the proof. Particularly, for $h=1$, we have:
\begin{equation*}
    \mathbb{E}_{\pi} \left[
    \left( \sum_{h = 1}^H r(s_h, a_h) - V_1^{\pi}(s_1) \right)^2
\right]=\sum_{h=1}^H \sum_{s,a} p_h^{\pi}(s, a) \mathrm{Var}_{p} \left( V_{h+1}^{\pi} \right)(s, a).
\end{equation*}
\end{proof}
\begin{lemma}[Variance bound for change of measure \citep{menard_fast_2021}] \label{lem:var-pq}
Let $ p, q \in \Sigma_S $ and $f $ is a function defined on $ S $ such that $ 0 \leq f(s) \leq b $ for all $ s \in S $. If $ \mathrm{KL}(p, q) \leq \alpha $ then
\begin{equation*}
    \mathrm{Var}_q(f) \leq 2 \mathrm{Var}_p(f) + 4 b^2 \alpha \quad \text{and} \quad \mathrm{Var}_p(f) \leq 2 \mathrm{Var}_q(f) + 4 b^2 \alpha.
\end{equation*}
\end{lemma}
\begin{lemma}[variance bound for change of function \citep{menard_fast_2021}]\label{lem:var-fg}
For $p, q \in \Sigma_S$, for $f, g$ two functions defined on $\mathcal{S}$ such that $0 \leq g(s), f(s) \leq b$ for all $s \in \mathcal{S}$, we have that
\begin{align}
    \text{Var}_p(f) &\leq 2 \text{Var}_p(g) + 2 b p |f - g| \nonumber \\
    \text{Var}_q(f) &\leq \text{Var}_p(f) + 3 b^2 \|p - q\|_1 \nonumber
\end{align}
where we denote the absolute operator by $|f|(s) = |f(s)|$ for all $s \in \mathcal{S}$.
\end{lemma}
\begin{lemma}[Summation to Integration \text{\citep[Lemma 9]{menard_fast_2021}}]
\label{lem:summation}
Let $ a_t $ be a sequence taking values in $[0 ,1]$ and $A_T\triangleq \sum_{t=0}^T a_t$ then
\begin{equation*}
    \sum_{t=0}^{T} \frac{a_{t+1}}{A_{T}\lor 1} \leq 4\log(A_{T+1}+1).
\end{equation*}
\end{lemma}
\begin{lemma}[Transcendental inequality to polynomial inequality]
\label{lem:solveineq}
Let $ A, B, C, D, E $, and $ \alpha $ be positive scalars such that $ 1 \leq B \leq E $ and $ \alpha \geq e $. If $ \tau \geq 1 $ satisfies
\begin{equation}\label{eq:solvetau}
    \tau \leq C \sqrt{\tau (A \log(\alpha \tau) + B (\log(\alpha \tau))^2)} + D (A \log(\alpha \tau) + E (\log(\alpha \tau))^2),
\end{equation}
then
\begin{equation*}
    \tau \leq C^2 (AF+BF^2) + \left(D+ 2C\sqrt{D} \right) (AF+EF^2)+1,
\end{equation*}
where
\begin{equation*}
    F = 4\log \left(2\alpha(A + E)(C + D) \right).
\end{equation*}
\end{lemma}
\begin{proof}
    Intuitively, the leading order of $\tau$ should be $\tilde{O}(C^2)$ since $\log(x)\le \frac{x^k}{k}$ for any $k>0$ and $x\ge 1$. To obtain this result, we first get a loose bound of $\tau$ to upper bound $\log(\tau)$. By doing this, we can transform the original inequality into a polynomial inequality, which is easier to solve. Following this idea, by choosing $k=1/2,1/4,3/4,3/8$, we obtain that:
    \begin{align}
        (A \log(\alpha \tau) + B (\log(\alpha \tau))^2)&\le 2A(\alpha \tau)^{\frac{1}{2}}+16B(\alpha \tau)^{\frac{1}{2}} \nonumber \\
        &\le (2A+16E)(\alpha \tau)^{\frac{1}{2}} \nonumber\\
        (A \log(\alpha \tau) + E (\log(\alpha \tau))^2)&\le \frac{4}{3}A(\alpha \tau)^{\frac{3}{4}}+\frac{64}{9}E(\alpha \tau)^{\frac{3}{4}} \nonumber \\
        &\le (2A+16E)(\alpha \tau)^{\frac{3}{4}}
    \end{align}
    This leads to:
    \begin{align}
        \tau&\le C\sqrt{(2A+16E)}\alpha^\frac{1}{4}\tau^\frac{3}{4}+D(2A+16E)(\alpha\tau)^\frac{3}{4} \nonumber \\
        &\le (C+D)(2A+16E)(\alpha\tau)^\frac{3}{4} \nonumber
    \end{align}
    This yields that:
    \begin{align}
        \alpha\tau&\le((C+D)(2A+16E))^4\alpha^4 \nonumber \\
        &\le (2\alpha(C+D)(A+B))^4
    \end{align}
    Here we define $F\triangleq 4\log(2\alpha(A+B)(C+D))$ for brevity, and substitute back to equation \ref{eq:solvetau}, we get:
    \begin{align}
        \tau&\le C\sqrt{(AF+BF^2)\tau}+D(AF+EF^2) \nonumber 
    \end{align}
    Solving this, we get:
    \begin{align}
        \sqrt{\tau}&\le \frac{C\sqrt{(AF+BF^2)}+\sqrt{C^2(AF+BF^2)+4D(AF+EF^2)}}{2} \nonumber \\
        &\le C\sqrt{AF+BF^2}+\sqrt{D}\sqrt{AF+EF^2} \nonumber \tag{$\sqrt{x+y}\le \sqrt{x}+\sqrt{y}$}
    \end{align}
    Finally, this gives us the result:
    \begin{equation}
        \tau\le C^2(AF+BF^2)+(D+2C\sqrt{D})(AF+EF^2)+1
    \end{equation}
\end{proof}
\section{Experiment Setup}\label{appendix:experiment}
We compare our algorithm with the state-of-the-art pure online RL algorithm BPI-UCBVI in \cite{menard_fast_2021} on GridWorld environment ($S=16,A=4,H=20$). The goal is to navigate in a room to collect rewards. In the source and the target environments, the same structure includes:
\begin{itemize}
    \item state-action space: the state space is a $4\times 4$ room, and the action space is to go up/down/left/right.
    \item horizon: each episode has a horizon length 20.
    \item success probability, the agent may fail in taking an action and go to the wrong direction with uniform probalibities. The success probability is set to be $0.95$ in experiment 1.
    \item reward: $r=1$ at state $(1,4)$, $r=0.1$ at state $(2,3)$, $r=0.01$ at state $(3,2)$, and $r=1.5$ at state $(3,4)$. The state $(1,4)$ is an absorbing state, where the agent cannot escape once steps in and the reward can only be obtained once.
    \item initial state: the agent starts from state $(3,2)$ in each episode.
\end{itemize}
Compared with the source environment, the target environment includes additional "traps" (absorbing states), at states $(2,2)$, $(2,4)$ and $(3,3)$, where the agent cannot escape once steps in. For experiment 1 and 2, the source dataset is collected by running \cref{alg:rf} in the source environment for $T=1\times 10^5$ episodes, which satisfies the condition in \cref{theorem:sample-complexity}. For both algorithms, $\varepsilon=0.1$, $\delta=0.1$. We re-scale the exploration bonus in BPI-UCBVI and \cref{alg:ucbvi} with the same constant $2\times 10^{-3}$ to mitigate the effect of the large hidden constant within $\widetilde{O}(\cdot)$ (similarlily for \cref{alg:rf} with $1\times 10^{-6}$). The optimality gap of a policy in the target environment is evaluated by running the policy for 100 episodes and calculating the average the results.

For experiment 1, we run both algorithms in the target environment for $T=2\times 10^5$ episodes to examine the relationship between optimality gaps and the sample size from the target environment. We set $\beta=0.45$ and $\sigma=0.25$ for \cref{alg:hybrid}, satisfying \cref{definition:separation,definition:reachability}. 

For experiment 2, we vary the success probability of taking an action in the target environment (not accounted in the implementation of \cref{alg:hybrid}) to examine the effect of maximum unidentified shift degree. The real success probability is set from $0.9$ to $0.55$ with a step size $0.05$. Because $\beta$ is still set to be $0.45$ for \cref{alg:hybrid}, the maximum unidentified shift degree (real $\beta$) ranges from $0.05$ to $0.4$ with a step size 0.05. For each success probability, we run the algorithms for 5 runs, each run contains $T=1\times 10^5$ episodes.
\end{document}